\documentclass{article} 

\usepackage[preprint]{icml2026}

\usepackage[american]{babel}

\usepackage{natbib}
\usepackage{mathtools}
\usepackage{amsfonts}
\usepackage{booktabs}
\usepackage{graphicx}
\usepackage{tikz}

\usepackage{amsthm}
\newtheorem{theorem}{Theorem}[section]
\newtheorem{lemma}[theorem]{Lemma}
\newtheorem{definition}[theorem]{Definition}
\newtheorem{proposition}[theorem]{Proposition}
\newtheorem{corollary}[theorem]{Corollary}
\newtheorem{example}[theorem]{Example}

\newcommand{\R}{\mathbb{R}}
\newcommand{\E}{\mathbb{E}}

\newcommand{\bX}{\boldsymbol{X}}
\newcommand{\onehot}{\mathrm{onehot}}
\newcommand{\titredupaper}{Decomposing Probabilistic Scores: Reliability, Information Loss and Uncertainty}

\icmltitlerunning{Decomposing Probabilistic Scores}

\begin{document}

\twocolumn[
  \icmltitle{Decomposing Probabilistic Scores\\ Reliability, Information Loss and Uncertainty}

\begin{icmlauthorlist}
    \icmlauthor{Arthur Charpentier}{xxx,zzz}
    \icmlauthor{Agathe Fernandes~Machado}{xxx}
  \end{icmlauthorlist}

  \icmlaffiliation{xxx}{UQAM, Montréal, Canada}
  \icmlaffiliation{zzz}{Kyoto University, Japan}
  \icmlcorrespondingauthor{Arthur Charpentier}{charpentier.arthur@uqam.ca}
 \vskip 0.3in
]
\printAffiliationsAndNotice{}

\begin{abstract}
Calibration is a conditional property that depends on the information retained by a predictor. We develop decomposition identities for arbitrary proper losses that make this dependence explicit. At any information level $\mathcal A$, the expected loss of an $\mathcal A$-measurable predictor splits into a proper-regret (reliability) term and a conditional entropy (residual uncertainty) term.
For nested levels $\mathcal A\subseteq\mathcal B$, a chain decomposition quantifies the information gain from $\mathcal A$ to $\mathcal B$. 
Applied to classification with features $\bX$ and score $S=s(\bX)$, this yields a three-term identity: miscalibration, a {\em grouping} term measuring information loss from $\bX$ to $S$, and irreducible uncertainty at the feature level. We leverage the framework to analyze post-hoc recalibration, aggregation of calibrated models, and stagewise/boosting constructions, with explicit forms for Brier and log-loss.


\end{abstract}

\section{Introduction}
\label{sec:intro}

Probabilistic classifiers are increasingly deployed in settings where predicted probabilities directly drive downstream decisions (screening, triage, pricing, resource allocation).
In such pipelines, a number interpreted as a probability should support decision-theoretic reasoning, which motivates {\em calibration}: informally, among instances assigned probability $s$, the empirical frequency of the event should be close to $s$.
At the same time, models are trained and compared using {\em proper scoring rules}, which combine calibration and sharpness/refinement into a single expected loss.

A central message of this paper is that calibration is inherently {\em relative to an information level}.
A predictor typically compresses features $\bX$ into a lower-dimensional score $S=s(\bX)$.
Even if one post-processes $S$ to make it calibrated, the compression $\bX \mapsto S$ can cause unavoidable information loss \citep{tasche2021calib}, invisible to calibration diagnostics alone.
Conversely, ensembling procedures (averaging, stacking, stagewise boosting) can improve an overall proper score while creating or amplifying miscalibration.
We develop decomposition identities for arbitrary proper losses making these trade-offs explicit and actionable.

\subsection{Setup and notation}
\label{sec:setup}

We work on a probability space $(\Omega,\mathcal F,\mathbb P)$.
The outcome $Y:\Omega\to\mathcal Y$ takes values in a {\em finite} label set (binary $\mathcal Y=\{0,1\}$, multi-class $\mathcal Y=\{1,\dots,K\}$), and features are given by $\bX:\Omega\to\mathcal X$. We denote by $\Delta(\mathcal Y)$ the probability simplex.

A probabilistic classifier is a measurable map $s:\mathcal X\to\Delta(\mathcal Y)$, and write $S:=s(\bX)\in\Delta(\mathcal Y)$ the induced random prediction.
Let $\mathcal H_{\bX}:=\sigma(\bX)$ and $\mathcal H_S:=\sigma(S)\subseteq \mathcal H_{\bX}$.

For any sub-$\sigma$-algebra $\mathcal A\subseteq\mathcal F$, we define the conditional law at level $\mathcal A$ by
\[
Q_{\mathcal A}:=\mathbb P(Y\in\cdot\mid \mathcal A)\in\Delta(\mathcal Y),
\]
understood as a $\Delta(\mathcal Y)$-valued $\mathcal A$-measurable random variable (up to a.s.\ equality).
We will repeatedly use
\[
\begin{cases}
    Q:=Q_{\mathcal H_{\bX}}=\mathbb P(Y\in\cdot\mid \bX),
\\
C:=Q_{\mathcal H_S}=\mathbb P(Y\in\cdot\mid S),
\end{cases}
\]
where $Q$ is the {\em true conditional law at the feature level} and $C$ is the {\em calibrated version} of the score.
Perfect calibration is the identity $S=C$ almost surely (equivalently $\mathbb P(Y\in\cdot\mid S)=S$ a.s.).

We evaluate probabilistic predictions with a loss $\ell:\Delta(\mathcal Y)\times\mathcal Y\to\mathbb R$ and its conditional risk (or {\em expected score})
\[
L(p,q):=\mathbb E_{Y\sim q}\big[\ell(p,Y)\big],\qquad p,q\in\Delta(\mathcal Y).
\]
When $\ell$ is proper, it induces a generalized entropy $\mathcal E_\ell(q):=L(q,q)$ and a nonnegative regret/divergence $d_\ell(p,q):=L(p,q)-\mathcal E_\ell(q)$.
Canonical examples include log-loss (yielding Kullback--Leibler divergence) and the Brier score (yielding squared Euclidean distance).

\subsection{Related work}

Proper scoring rules are classical tools for probabilistic forecasting and decision-theoretic evaluation \citep{gneiting2007probabilistic}.
Decomposition ideas trace back to Murphy-type reliability--resolution--uncertainty identities and extensions to proper scores and multi-class settings \citep{brocker2009reliability}.
In machine learning, post-hoc recalibration methods such as Platt scaling and isotonic regression are widely used \citep{platt1999probabilistic}, and recent work emphasizes how calibration interacts with training dynamics, early stopping, and refinement--calibrate phenomena \citep{kullflach2015novel,dimitriadis2023triptych,berta2025earlystopping}.
Ensembling and stagewise procedures complicate calibration: they may improve expected proper losses while degrading conditional reliability, motivating analysis beyond a single score.
Our key contribution is not the tower property, but an {\em information-level} view of proper-score decompositions that separates miscalibration from {\em information loss} induced by compression (e.g., $\bX\mapsto S$), a phenomenon invisible to calibration curves.

\subsection{Contributions and organization}

\paragraph{Contributions.}
Our contributions are as follows.
{\setlength{\leftmargini}{0em}
\setlength{\itemsep}{0pt}
\setlength{\topsep}{0pt}
\setlength{\parsep}{0pt}
\begin{itemize}
\item[-] \textbf{Information-level decompositions for arbitrary proper losses.}
We derive level and chain identities for $\mathcal A\subseteq\mathcal B$ that split regret from conditional entropy. When applied to $\mathcal H_S\subseteq\mathcal H_{\bX}$, this specializes to the decomposition of $\E[\ell(S,Y)]$ used in \cite{LebelMV23} into miscalibration, a {\em grouping} term measuring information loss $\bX\to S$, and irreducible feature-level uncertainty.
\item[-] \textbf{Implications for recalibration and ensembles.}
We view recalibration as projection onto $\sigma(S)$ and characterize when aggregation breaks calibration and how stagewise/boosting trades reliability for refinement.
\item[-] \textbf{Concrete instantiations.}
We derive closed forms for the Brier score and log-loss.
\end{itemize}}

While our proofs rely on standard conditional expectation arguments, the contribution is an operational separation of three failure modes under {\em any} proper loss: (i) \textbf{reliability} (proper regret at the score level), (ii) \textbf{information loss} induced by compression $\boldsymbol X\!\to\!S$ (grouping), and (iii) \textbf{irreducible} uncertainty at the feature level. This yields actionable guidance: post-hoc recalibration can only reduce (i); to reduce (ii) one must change the score/representation (retain more information);
and (iii) can only be reduced by richer features or better data.

\paragraph{Organization.}
Section~\ref{sec:decompositions} presents the one-level and chain decompositions and interprets them for probabilistic classification.
Section~\ref{sec:well-calibrated} discusses global balance versus calibration and practical diagnostics.
Section~\ref{sec:recalibration} studies post-hoc recalibration under proper losses.
Section~\ref{sec:calib-ensembles} covers aggregation and stagewise/boosting constructions and illustrates the framework
on Brier and log-loss.

\section{Decompositions}\label{sec:decompositions} 

\subsection{General Decompositions}

\begin{theorem}\label{thm:onelevel}
Let $\mathcal A\subseteq\mathcal F$ and let $T$ be a $\mathcal A$-measurable random variable taking values in $\Delta(\mathcal Y)$. Then
\[
\mathbb E\big[\ell(T,Y)\big]
=
\mathbb E\big[d_\ell(T,Q_{\mathcal A})\big]
+
\mathbb E\big[\mathcal{E}_\ell(Q_{\mathcal A})\big].
\]
\end{theorem}

The first term $\mathbb E[d_\ell(T,Q_{\mathcal A})]$ is a {\em measurable regret} at information level $\mathcal A$ (and becomes a calibration/reliability term when $\mathcal A=\mathcal H_T$).
The second term $\mathbb E[\mathcal E_\ell(Q_{\mathcal A})]$ is the {\em residual uncertainty} (or Bayes risk) given $\mathcal A$.
When $\ell$ is (strictly) proper, $d_\ell(\cdot,\cdot)\ge 0$ (and $d_\ell(p,q)=0$ iff $p=q$), so each regret term quantifies a genuine loss relative to the best predictor available at that information level.

\begin{theorem}[Chain decomposition for nested information levels]
\label{thm:chain}
Let $\mathcal A\subseteq \mathcal B\subseteq \mathcal F$ and let $T$ be a $\mathcal A$-measurable random variable taking values in $\Delta(\mathcal Y)$. Then
\[
\mathbb E[\ell(T,Y)]
=
\mathbb E\big[d_\ell(T,Q_{\mathcal A})\big]
+\mathbb E\big[d_\ell(Q_{\mathcal A},Q_{\mathcal B})\big]
+\mathbb E\big[\mathcal E_\ell(Q_{\mathcal B})\big].
\]
\end{theorem}

\subsection{Back to Calibration}

Take $T=S$, $\mathcal A=\mathcal{H}_S$, and $\mathcal B=\mathcal{H}_{\bX}$.
Then $Q_{\mathcal A}=C$ and $Q_{\mathcal B}=Q$, hence
\begin{align}
\mathbb{E}[\ell(S,Y)]
& =
\underbrace{\mathbb{E}[d_\ell(S,C)]}_{\text{calibration / reliability}}
+
\underbrace{\mathbb{E}[d_\ell(C,Q)]}_{\text{grouping (information loss from $\bX$ to $S$)}}\notag\\
&+
\underbrace{\mathbb{E}[\mathcal E_\ell(Q)]}_{\text{irreducible at the $\bX$ level}}.\label{eq:clf-chain}
\end{align}
Equation~\eqref{eq:clf-chain} is reminiscent of classical reliability--resolution--uncertainty decompositions, but here it arises as a special case of a general
{\em information-level chain rule} valid for arbitrary proper losses, and the middle term admits a direct interpretation as the {\em information loss}
by the compression $\boldsymbol X\!\to\!S$, which is precisely what governs what recalibration and ensemble projections can (and cannot) fix, as mentioned in \cite{tasche2021calib, LebelMV23}.

\begin{lemma}[When does the grouping term vanish?]\label{rem:grouping-zero}
Assume $\ell$ is strictly proper. Then $\mathbb E[d_\ell(C,Q)]=0$ if and only if $C=Q$ almost surely.
Equivalently, the ``true'' conditional law $Q=\mathbb P(Y\in\cdot\mid \bX)$ is $\sigma(S)$-measurable
(i.e., $S$ is sufficient for $\bX$ with respect to predicting $Y$).
\end{lemma}

\begin{example}[Perfect calibration does not prevent large information loss]
\label{ex:calibrated-yet-grouping}
Consider binary $Y\in\{0,1\}$. Let $X\in\{1,2\}$ with $\mathbb P(X=1)=\mathbb P(X=2)=1/2$,
and
\[
\mathbb P(Y=1\mid X=1)=0.9,\qquad \mathbb P(Y=1\mid X=2)=0.1.
\]
Define the {\em constant} score $S\equiv 1/2$. Then $C=\mathbb E[Y\mid S]=1/2=S$,
so $S$ is perfectly calibrated and the reliability term vanishes.

However, the true conditional law $Q=\mathbb E[Y\mid X]$ takes values $0.9$ or $0.1$,
so the grouping term is strictly positive. For the Brier loss,
\[
\mathbb E[(C-Q)^2]=\tfrac12(0.5-0.9)^2+\tfrac12(0.5-0.1)^2=0.16,
\]
and for the log-loss,
\begin{align*}
\mathbb E[\mathrm{KL}(Q\|C)]
&=
\tfrac12\mathrm{KL}(\mathrm{Bern}(0.9)\|\mathrm{Bern}(0.5))\\
&+\tfrac12\mathrm{KL}(\mathrm{Bern}(0.1)\|\mathrm{Bern}(0.5)).    
\end{align*}
Thus, calibration alone is compatible with severe information loss (no discrimination).
\end{example}

Post-hoc recalibration ideally reduces the term $\mathbb{E}[d_\ell(S,C)]$, but without changing discrimination, it cannot directly reduce the grouping term $\mathbb{E}[d_\ell(C,Q)]$.
This tension lies at the heart of recent analyses of training dynamics and early stopping~\citep{berta2025earlystopping}.
\cite{brocker2009reliability} generalizes the reliability/resolution decomposition to the multi-class case and to proper scores.
In our loss-based formalism, one may write an ``uncertainty + reliability $-$ resolution'' form by introducing the marginal $\bar Q := \mathbb{P}(Y\in\cdot)$ (coined the ``climatology'' in \cite{gneiting2007probabilistic}).

\begin{proposition}[Uncertainty--resolution--reliability form]
\label{prop:urc}
Let $S$ be a prediction and let $C=\mathbb{P}(Y\in\cdot\mid S)$.
Then
\[
\mathbb{E}[\ell(S,Y)]
=
\mathcal{E}_\ell(\bar Q)
-\mathbb{E}\big[d_\ell(\bar Q, C)\big]
+\mathbb{E}\big[d_\ell(S,C)\big],
\]
where $\mathbb{E}[d_\ell(\bar Q,C)]$ plays the role of {\em resolution} (information gain relative to climatology) and $\mathbb{E}[d_\ell(S,C)]$ that of {\em reliability}.
\end{proposition}

\subsection{Fundamental uncertainty}

We introduce a latent variable $Z$ (state of the world) and the {\em objective chance}
\[
\Pi := \mathbb{P}(Y\in\cdot \mid Z)\in \Delta(\mathcal Y).
\]
The object $\Pi$ formalizes the idea of a ``true probability'' that may vary across individuals/situations even at fixed $\bX$ (unobserved covariates, non-stationarity, partially observed causal mechanisms, etc.).

Taking the nested chain $\mathcal H_S\subseteq \mathcal H_{\bX}\subseteq \mathcal H_Z$,
Theorem~\ref{thm:chain} can be iterated and yields a four-term decomposition.

\begin{theorem}
\label{thm:four}
Assume $\mathcal H_S\subseteq \mathcal H_{\bX}\subseteq \mathcal H_Z$. Then
\begin{align*}
\mathbb{E}[\ell(S,Y)]
&=
\mathbb{E}[d_\ell(S,C)]
+\mathbb{E}[d_\ell(C,Q)]\\
&+\mathbb{E}[d_\ell(Q,\Pi)]
+\mathbb{E}[\mathcal{E}_\ell(\Pi)].
\end{align*}
\end{theorem}

We have here the following interpretations,
\begin{itemize}
\item $\mathbb{E}[d_\ell(S,C)]$: {\em miscalibration} (lack of reliability).
\item $\mathbb{E}[d_\ell(C,Q)]$: {\em grouping} / information loss induced by the map $\bX\mapsto S$ (even if perfectly calibrated, $S$ may still ``collapse'' heterogeneity).
\item $\mathbb{E}[d_\ell(Q,\Pi)]$: {\em chance heterogeneity} remaining at the $\bX$ information level (the true probability is not identified by $\bX$).
\item $\mathbb{E}[\mathcal{E}_\ell(\Pi)]$: {\em intrinsic noise} at the $Z$ level (what cannot be reduced even with $Z$).
\end{itemize}
This clarifies that the ``irreducible'' component is relative to the observability level (here $\bX$), as already discussed in epistemic/aleatoric decompositions~\citep{kullflach2015novel}.

Let $k=2$ and identify $p\in\Delta(\mathcal Y)$ with $p=\mathbb{P}(Y=1)$.
We recover the same results as in \cite{tasche2021calib} when considering the special case of the Brier loss, $\ell(p,y)=(p-y)^2$.

\begin{corollary}[Variance-type decomposition (Brier)]
\label{cor:brier}
In the binary case, with $C=\mathbb{E}[Y\mid S]$ and $Q=\mathbb{E}[Y\mid \bX]$,
\[
\mathbb{E}[(S-Y)^2] = \mathbb{E}[(S-C)^2] + \mathbb{E}[(C-Q)^2] + \mathbb{E}[Q(1-Q)].
\]
If moreover $\Pi=\mathbb{E}[Y\mid Z]$, then
\[
\mathbb{E}[Q(1-Q)] = \mathbb{E}[(Q-\Pi)^2]+\mathbb{E}[\Pi(1-\Pi)].
\]
\end{corollary}

In the binary Brier case, $Q=\mathbb E[Y\mid X]$ and $C=\mathbb E[Y\mid S]$ satisfy
$C=\mathbb E[Q\mid S]$. Hence
\[
\mathbb E[(C-Q)^2]
=
\mathbb E\left[\mathrm{Var}(Q\mid S)\right],
\]
i.e., the grouping term is the expected conditional variance of the Bayes score $Q$ after compressing $\bX\mapsto S$. Equivalently, by the law of total variance,
\[
\mathbb E[C(1-C)] = \mathbb E[Q(1-Q)] + \mathbb E[(C-Q)^2],
\]
so moving from $\mathcal H_X$ to $\mathcal H_S$ increases residual uncertainty exactly by the information-loss term $\mathbb E[(C-Q)^2]$.

\begin{corollary}[Log-loss: entropies and information]
\label{cor:logloss-info}
Assume the log-loss $\ell_{\log}(p,y):=-\log p_y$ on a finite label set.
Then the induced divergence and generalized entropy are $d_{\log}(p,q)=\mathrm{KL}(q\|p)$ and $\mathrm{E}_{\log}(q)=H(q)$.

For any $\mathcal A\subseteq \mathcal B\subseteq \mathcal F$ and any
$\mathcal A$-measurable predictor $T:\Omega\to\Delta(\mathcal Y)$,
\[
\mathbb E[\ell_{\log}(T,Y)]
=
\mathbb E\left[\mathrm{KL}(Q_{\mathcal A}\|T)\right]
+
\mathbb E\left[\mathrm{KL}(Q_{\mathcal B}\|Q_{\mathcal A})\right]
+
H(Y\mid \mathcal B),
\]
where $H(Y\mid \mathcal B):=\mathbb E\left[H(Q_{\mathcal B})\right]$.
Moreover,
\[
\mathbb E\left[\mathrm{KL}(Q_{\mathcal B}\|Q_{\mathcal A})\right]
=
I(Y;\mathcal B\mid \mathcal A).
\]

In particular, with $\mathcal A=\mathcal H_S$ and $\mathcal B=\mathcal H_X$,
\[
\mathbb E[-\log S_Y]
=
\mathbb E[\mathrm{KL}(C\|S)]
+
I(Y;X\mid S)
+
H(Y\mid X).
\]
\end{corollary}

For the log-loss (cross-entropy), $d$ becomes a Kullback--Leibler divergence and $\mathcal{E}_{\ell}$ the Shannon entropy; the chain decomposition therefore links the expected loss to information-theoretic quantities.
This motivates diagnostics (Murphy diagrams, Murphy decomposition, etc.) and the ``triptych'' view of calibration/discrimination/performance~\citep{dimitriadis2023triptych}.


\section{Well-calibration}
\label{sec:well-calibrated}
\subsection{Global balance (marginal calibration).}
We call global balance the weak moment condition $\E[S]=\E[\onehot(Y)]$ (binary: $\E[S]=\E[Y]$), where $\onehot(Y)\in\Delta(\mathcal Y)$ is the one-hot encoding of $Y$.
It is strictly weaker than calibration since it does not constrain the conditional relationship between $S$ and $Y$; in particular, intercept-corrections can enforce global balance on a validation split without guaranteeing $\mathbb{P}(Y\in\cdot\mid S)=S$.
We therefore focus on conditional calibration and proper-loss reliability.

\subsection{Calibrated scores and a proper-loss notion of calibration error}
\label{sec:well-calibration-def}

Recall $Q:=\mathbb P(Y\in\cdot\mid \bX)$ and $C:=\mathbb P(Y\in\cdot\mid S)$.
Perfect calibration holds when $S=C$ almost surely.

\begin{definition}[Perfect calibration]
\label{def:perfect-calibration}
We say that $S$ is {\em perfectly calibrated} if
\[
\mathbb P(Y\in\cdot\mid S)=S \qquad\text{a.s.},
\]
equivalently $C=S$ a.s.
In the binary case, this is $\mathbb E[Y\mid S]=S$ a.s.
\end{definition}

Proper losses provide a canonical {\em calibration error} (reliability) through the induced regret.
Recall $L(p,q):=\mathbb E_{Y\sim q}[\ell(p,Y)]$, $\mathcal E_\ell(q):=L(q,q)$, and
$d_\ell(p,q):=L(p,q)-\mathcal E_\ell(q)$.

\begin{proposition}[Calibration error as proper regret]
\label{prop:calib-regret}
Let $\ell$ be a proper loss. For any score $S$ and $C=\mathbb P(Y\in\cdot\mid S)$,
\[
\mathbb E[\ell(S,Y)]
=
\mathbb E\big[d_\ell(S,C)\big]+\mathbb E\big[\mathcal E_\ell(C)\big].
\]
In particular, $\mathbb E[d_\ell(S,C)]\ge 0$, and if $\ell$ is strictly proper then $\mathbb E[d_\ell(S,C)]=0$ iff $S=C$ a.s.
Moreover, perfect calibration implies global balance: $\mathbb E[S]=\mathbb E[\onehot(Y)]$.
\end{proposition}

\subsection{Diagnostics and visualization}
\label{sec:calib-diagnostics}

Empirical calibration assessment amounts to estimating the map $s\mapsto \mathbb P(Y\in\cdot\mid S=s)$ (or its binary analogue).
Reliability diagrams (binning $S$ and plotting empirical frequencies) are intuitive but sensitive to bin choices.
Stable alternatives use a fixed estimation rule (e.g. isotonic regression in the binary case \citep{dimitriadis2021stable}), and can be combined with proper-score-based comparisons (e.g. Murphy-type diagrams and the ``triptych'' perspective) to separate reliability from refinement and overall performance.
In our notation, these diagnostics primarily target the reliability term $\mathbb E[d_\ell(S,C)]$ and help distinguish miscalibration from information loss (grouping) and irreducible uncertainty.

Given test data $(S_i,X_i,Y_i)_{i=1}^n$, we estimate $C(s)=P(Y\in\cdot\mid S=s)$ by a recalibration fit $\widehat C(s)=\widehat g(s)$ 
(estimated out-of-sample using an independent calibration split, or via cross-fitting). We then estimate reliability by
\[
\widehat{\mathrm{Rel}}_\ell:=\frac1n\sum_{i=1}^n d_\ell(S_i,\widehat C(S_i)).
\]
When a proxy $\widehat Q(X)\approx P(Y\in\cdot\mid X)$ is available (oracle in simulations, or a high-capacity reference model), we analogously estimate grouping and irreducible terms by
$\widehat{\mathrm{Grp}}_\ell:=\frac1n\sum_i d_\ell(\widehat C(S_i),\widehat Q(X_i))$ and
$\widehat{\mathrm{Irr}}_\ell:=\frac1n\sum_i \mathrm{E}_\ell(\widehat Q(X_i))$. \cite{LebelMV23} also propose estimating the lower bound of the grouping loss.


\section{Recalibration}\label{sec:recalibration}

A (post-hoc) recalibrator produces an updated score $\widetilde S$ that is measurable with respect to the original score information $\sigma(S)$.
Equivalently, it restricts attention to predictors of the form $\widetilde S=g(S)$ for a measurable map
$g:\Delta(\mathcal Y)\to\Delta(\mathcal Y)$, so that
\[
\widetilde S := g(S),\qquad \sigma(\widetilde S)\subseteq \sigma(S).
\]
In the binary case ($S\in[0,1]$), one often restricts to non-decreasing $g$ to preserve the ranking induced by $S$.

\paragraph{Practical note.}
In practice $C=\mathbb P(Y\in\cdot\mid S)$ is unknown and $g$ is learned on held-out calibration data (or via cross-fitting); fitting $g$ on the same data used to train the classifier can lead to overly optimistic calibration assessments.

\begin{proposition}[Population-optimal recalibration \citep{berta2025earlystopping}]
\label{prop:pop-opt-recal}
Assume $\ell$ is proper and consider the class of predictors that are $\sigma(S)$-measurable, i.e. of the form $\widetilde S=g(S)$.
Then the population risk is minimized by $\widetilde S^\star=C$:
\[
\inf_{g:\Delta(\mathcal Y)\to\Delta(\mathcal Y)} \big\lbrace\mathbb{E}[\ell(g(S),Y)] \big\rbrace =\mathbb{E}[\ell(C,Y)] = \mathbb{E}[\mathcal E_\ell(C)].
\]
Moreover, the excess risk of any recalibrator $g$ admits the decomposition
\[
\mathbb{E}[\ell(g(S),Y)]-\mathbb{E}[\ell(C,Y)] = \mathbb{E}\big[d_\ell(g(S),C)\big].
\]
\end{proposition}

\paragraph{What recalibration can and cannot fix.}
By Proposition~\ref{prop:pop-opt-recal}, post-hoc recalibration can only improve the {\em reliability} term (it acts within the information level $\sigma(S)$).
It cannot, without changing the underlying score, directly reduce the {\em grouping} term $\mathbb E[d_\ell(C,Q)]$ in~\eqref{eq:clf-chain}, which quantifies information loss from $\bX$ to $S$.

\paragraph{Why require $g$ to be non-decreasing in the binary case?}
Beyond reliability, many downstream uses of a score rely on {\em ranking} rather than on the absolute probability values: thresholding rules, ROC/AUC, or top-$k$ selection.
Recalibration should ideally fix reliability {\em without harming discrimination}.
A minimal compatibility constraint is that $g:[0,1]\to[0,1]$ be non-decreasing.

\begin{proposition}[Monotone recalibration preserves threshold rules and ROC (up to ties)]
\label{prop:monotone-preserves-roc}
Let $Y\in\{0,1\}$, let $S\in[0,1]$, and let $g:[0,1]\to[0,1]$ be non-decreasing.
Define $\widetilde S=g(S)$.
For any threshold $\tau\in[0,1]$, there exists $\tau'\in[0,1]$ such that
\[
\{\widetilde S\ge \tau\}=\{S\ge \tau'\}\quad\text{a.s.}
\]
Consequently, all threshold classifiers induced by $\widetilde S$ are also induced by $S$, and the ROC curve (hence AUC) is unchanged by $g$ up to tie-handling conventions.
\end{proposition}

This motivates isotonic regression and monotone smoothers as {\em order-consistent} calibrators: monotonicity preserves the weak order of $S$, while strict monotonicity on the support of $S$ additionally avoids large flat regions.

\paragraph{Finite-sample validity (context).}
Recalibration is typically justified as risk minimization under a proper loss (hence asymptotic/model-based guarantees).
A complementary line of work provides {\em finite-sample} validity under minimal assumptions, e.g.  conformal prediction for set-valued outputs and Venn--Abers predictors for interval-valued probabilities \citep{vovk2005algorithmic,vovk2012vennabers}.
These methods highlight that formal guarantees often attach to the full object produced (set/interval), and may not be preserved by collapsing it to a single probability.
This reinforces a recurring theme of our analysis: calibration is a conditional property and can be fragile under post-processing and aggregation.


\section{Calibration and Ensemble Learning}
\label{sec:calib-ensembles}

\subsection{Aggregation: calibrated inputs do not imply a calibrated output}
\label{sec:aggregation}

A common practice is to combine several probabilistic classifiers $S^{(1)},\dots,S^{(M)}\in\Delta(\mathcal Y)$ into an aggregated score $S^{\mathrm{ens}} = A(\mathbf S)$ where $\mathbf S:=(S^{(1)},\dots,S^{(M)})$ and $A$ is, e.g., averaging, stacking, or a learned combiner.
Even if each component is perfectly calibrated, calibration is generally {\em not} preserved by aggregation (see Proposition~\ref{prop:agg-preserve}, Figure~\ref{fig:averaging} and \cite{gneiting2013}).

Building on the chain decomposition from Theorem~\ref{thm:chain}, to understand when calibration can be preserved, define the conditional laws
\[
C_{\mathbf S}:=\mathbb P(Y\in\cdot\mid \mathbf S),
\qquad
C_{\mathrm{ens}}:=\mathbb P(Y\in\cdot\mid S^{\mathrm{ens}}).
\]
By definition, $S^{\mathrm{ens}}$ is calibrated if and only if $C_{\mathrm{ens}}=S^{\mathrm{ens}}$ almost surely.

\begin{proposition}[When does an aggregator preserve calibration?]
\label{prop:agg-preserve}
Let $S^{\mathrm{ens}}=A(\mathbf S)$ and define $C_{\mathbf S}$ and $C_{\mathrm{ens}}$ as above.
If there exists a measurable map $\varphi:\Delta(\mathcal Y)\to\Delta(\mathcal Y)$ s.t.
\[
C_{\mathbf S}=\varphi(S^{\mathrm{ens}})\quad\text{a.s.},
\]
then $C_{\mathrm{ens}}=\varphi(S^{\mathrm{ens}})$ a.s. and $S^{\mathrm{ens}}$ is calibrated if and only if $\varphi(s)=s$ for $s$ in the (essential) support of $S^{\mathrm{ens}}$.
\end{proposition}
In the binary case, consider the average $\bar S=\frac1M\sum_{m=1}^M S^{(m)}$ of $M$ individually calibrated scores.
Then $\bar S$ is calibrated only if $\mathbb P(Y=1\mid S^{(1)},\ldots,S^{(M)})$ depends on $(S^{(1)},\ldots,S^{(M)})$ {\em only through} $\bar S$ (i.e., $\bar S$ is sufficient for $Y$ under the projection $(S^{(1)},\ldots,S^{(M)})\mapsto \bar S$).
Otherwise, averaging can destroy calibration even when each component is perfectly calibrated (Example~\ref{ex:avg-breaks-calibration}).

\begin{example}[Two calibrated scores whose average is miscalibrated]
\label{ex:avg-breaks-calibration}
Let $Y\in\{0,1\}$ and let $(S^{(1)},S^{(2)})\in\{0.25,0.75\}^2$ take the four values $(0.25,0.25),(0.25,0.75),(0.75,0.25),(0.75,0.75)$ each with probability $1/4$.
Conditionally on $(S^{(1)},S^{(2)})$, let $Y\sim\mathrm{Bern}(m)$ with
\[
\begin{cases}
    m(0.25,0.25)=0,\quad &m(0.25,0.75)=0.5,\\
    m(0.75,0.25)=0.5,\quad &m(0.75,0.75)=1.
\end{cases}
\]
Then one checks that $\mathbb E[Y\mid S^{(1)}]=S^{(1)}$ and $\mathbb E[Y\mid S^{(2)}]=S^{(2)}$, so each component is perfectly calibrated.

However, for the average $\bar S:=\tfrac12(S^{(1)}+S^{(2)})$, we have $\mathbb P(\bar S=0.25)=1/4$ and $\mathbb E[Y\mid \bar S=0.25]=0\neq 0.25$, so $\bar S$ is not calibrated. This illustrates that a projection $S\mapsto A(S)$ can destroy calibration even when each coordinate is calibrated.
\end{example}

If one keeps the full vector $\mathbf S$, then the Bayes-optimal calibrated prediction at that information level is $C_{\mathbf S}=\mathbb P(Y\in\cdot\mid \mathbf S)$ (a ``stacking'' view).
If instead one outputs only $S^{\mathrm{ens}}$, then post-hoc recalibration targets $C_{\mathrm{ens}}=\mathbb P(Y\in\cdot\mid S^{\mathrm{ens}})$ but cannot recover information lost by the projection
$\mathbf S\mapsto S^{\mathrm{ens}}$.

\subsection{Sequential ensembles and boosting: refinement as increasing information}
\label{sec:boosting}

Boosting builds predictors sequentially by adding weak learners, denoted $h_t(\bX)$ at step $t$.
A convenient abstraction focuses on the {\em information revealed across stages}: let
\[
\mathcal{F}_t := \sigma\big(h_1(\bX),\dots,h_t(\bX)\big),
\]
with $\mathcal{F}_0\subseteq\mathcal{F}_1\subseteq\cdots\subseteq\mathcal{F}_T\subseteq\mathcal H_{\bX}$,
and define the Bayes-optimal prediction at level $\mathcal F_t$ by
\[
Q_t := \mathbb P(Y\in\cdot\mid \mathcal F_t)\in\Delta(\mathcal Y).
\]
Each $Q_t$ is perfectly calibrated (by construction) and becomes more refined as $t$ increases.

\begin{theorem}[Telescoping decomposition along a filtration (proper losses)]
\label{thm:boosting-telescope}
Let $\ell$ be a proper loss with generalized entropy $\mathcal E_\ell$ and induced divergence $d_\ell$.
Let $\mathcal F_0\subseteq \cdots \subseteq \mathcal F_T$ be a filtration and define $Q_t:=\mathbb P(Y\in\cdot\mid \mathcal F_t)$.
Then for any $\mathcal F_0$-measurable predictor $T:\Omega\to\Delta(\mathcal Y)$,
\[
\mathbb E[\ell(T,Y)]
=
\mathbb E[d_\ell(T,Q_0)]
+\sum_{t=0}^{T-1}\mathbb E[d_\ell(Q_t,Q_{t+1})]
+\mathbb E[\mathcal E_\ell(Q_T)].
\]
\end{theorem}

\begin{corollary}[Log-loss along a filtration: information gains]
\label{cor:filtration-logloss-info}
Under log-loss, increments in Theorem~\ref{thm:boosting-telescope} satisfy
\[
\mathbb E[d_{\log}(Q_t,Q_{t+1})]
=
\mathbb E[\mathrm{KL}(Q_{t+1}\|Q_t)]
=
I(Y;\mathcal F_{t+1}\mid \mathcal F_t),
\]
and therefore $\mathbb E[\ell_{\log}(T,Y)]$ equals
\[
\mathbb E[\mathrm{KL}(Q_0\|T)]
+
\sum_{t=0}^{T-1} I(Y;\mathcal F_{t+1}\mid \mathcal F_t)
+
H(Y\mid \mathcal F_T).
\]
\end{corollary}

\begin{corollary}[Brier case: Pythagoras identity]
\label{cor:boosting-brier}
Assume $\mathcal Y=\{0,1\}$ and $\ell(p,y)=(p-y)^2$.
Let $\eta_t:=\mathbb E[Y\mid\mathcal F_t]\in[0,1]$. Then
\[
\mathbb E[(Y-\eta_{t+1})^2]
=
\mathbb E[(Y-\eta_t)^2]
-\mathbb E[(\eta_{t+1}-\eta_t)^2],
\]
so the Brier risk decreases with $t$ and the decrease equals the refinement gain $\mathbb E[(\eta_{t+1}-\eta_t)^2]$.
\end{corollary}

\paragraph{Martingale view (Brier).}
The process $(\eta_t)_{t=0}^T$ is a bounded martingale with respect to $(\mathcal F_t)$ since
\[
\mathbb E[\eta_{t+1}\mid \mathcal F_t]
=
\mathbb E\big[\mathbb E[Y\mid \mathcal F_{t+1}]\mid \mathcal F_t\big]
=
\mathbb E[Y\mid \mathcal F_t]
=\eta_t.
\]
Corollary~\ref{cor:boosting-brier} is therefore a Doob--Pythagoras identity: each refinement step decomposes the mean squared error into a remaining error plus an orthogonal increment, and the gain $\mathbb E[(\eta_{t+1}-\eta_t)^2]$ measures how much additional information $\mathcal F_{t+1}$ brings beyond $\mathcal F_t$.
Equivalently,
\[
\mathbb E[(Y-\eta_t)^2]=\mathbb E[\mathrm{Var}(Y\mid \mathcal F_t)]
\]
and
\[
\mathbb E[(\eta_{t+1}-\eta_t)^2]=\mathbb E[\mathrm{Var}(Y\mid \mathcal F_t)-\mathrm{Var}(Y\mid \mathcal F_{t+1})],
\]
so refinement reduces conditional variance in expectation.

In practice, boosting does not produce the Bayes sequence $(Q_t)$: updates are constrained (e.g., additive logits, shallow trees) and optimization is approximate. Hence, even if weak learners are calibrated in isolation, the final predictor need not be calibrated.
A robust practice is to recalibrate {\em after} boosting, or to use calibration-aware objectives (cf.\ the refine--then--calibrate discussion in~\citep{berta2025earlystopping}).

\paragraph{Stagewise recalibration (population view).}
Let $\widetilde S_t\in\Delta(\mathcal Y)$ be the (possibly miscalibrated) score at stage $t$ and define its population recalibration
\[
S_t := \mathbb P(Y\in\cdot\mid \widetilde S_t).
\]
Then $S_t$ is perfectly calibrated with respect to its own output and (strictly) improves any strictly proper loss:
\[
\mathbb E[\ell(\widetilde S_t,Y)]
=
\mathbb E\big[d_\ell(\widetilde S_t,S_t)\big]
+\mathbb E[\ell(S_t,Y)].
\]
{\em Crucially}, the projection $\widetilde S_t\mapsto S_t$ can reduce miscalibration but may also reduce resolution if it collapses distinct situations to the same score value; therefore, calibrating at each stage is not automatically beneficial for discrimination, and a final recalibration step is often still required.


\section{Empirical illustrations}
\label{sec:experiments}

Our goal is not to optimize predictive performance, but to illustrate how the proposed decompositions separate miscalibration (reliability) from information loss (grouping) and residual uncertainty, and how this separation guides practical interventions such as recalibration and ensembling.

\subsection{Synthetic Data}



We use a binary data-generating process with two correlated covariates and known conditional probability $Q(\bX)=\mathbb P(Y=1\mid \bX)$; a correlation parameter $\rho$ controls dependence
($\bX$ is generated from a Gaussian copula with correlation $\rho$ and uniform marginals).
Full details and the true probability surface are given in Appendix~\ref{app:dgp}.

\paragraph{Base models and aggregation.}
On a training split, we fit two misspecified logistic models
\[
\widehat s_1:=\mathbb P_{\hat\theta_1}(Y=1\mid X_1),\text{ and }
\widehat s_2:=\mathbb P_{\hat\theta_2}(Y=1\mid X_2),
\]
and form their average $\widehat s_{\mathrm{ens}}:=\tfrac12(\widehat s_1+\widehat s_2)$.
On a separate calibration split, we learn monotone smooth recalibration maps $\widehat g_1,\widehat g_2,\widehat g_{\mathrm{ens}}:[0,1]\to[0,1]$ and define the recalibrated scores $S^{(1)}=\widehat g_1(\widehat s_1)$, $S^{(2)}=\widehat g_2(\widehat s_2)$ and $S^{(\mathrm{ens})}=\widehat g_{\mathrm{ens}}(\widehat s_{\mathrm{ens}})$, see Appendix~\ref{app:g:smooth}.

\paragraph{Local calibration score (LCS).}
To quantify calibration we use the local calibration score
\[
\mathrm{LCS}(S;\widehat g)
:=\frac{1}{n}\sum_{i=1}^n\Big(S_i-\widehat g(S_i)\Big)^2,
\]
calculated using local regression as in \cite{Austin2019TheIC}, on an independent test split.
Intuitively, $\mathrm{LCS}$ measures the squared distance between the score and its estimated calibrated version, and serves as a scalar summary complementary to calibration curves. For a proper loss $\ell$, the natural reliability estimate is $\widehat{\E}[d_\ell(S,\widehat C(S))]$; we use LCS only as a compact proxy.

\paragraph{Result.}
Figure~\ref{fig:averaging} shows that the two base models can be nearly calibrated (small LCS), while their average exhibits substantially larger miscalibration, especially as $\rho$ varies.
This empirically illustrates that calibration is not generally preserved under aggregation, even when each component
is (approximately) calibrated.

\begin{figure}[!htb]
    \centering
    \includegraphics[width=0.49\linewidth]{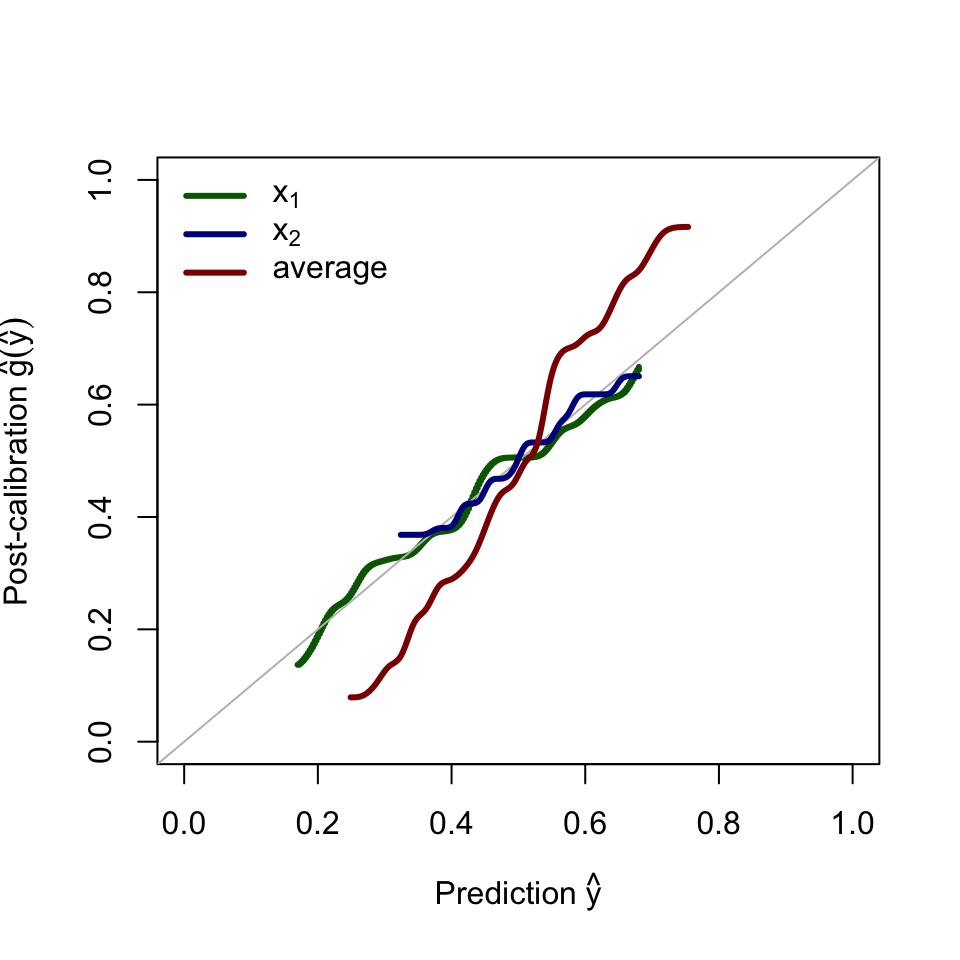}\hfill
    \includegraphics[width=0.49\linewidth]{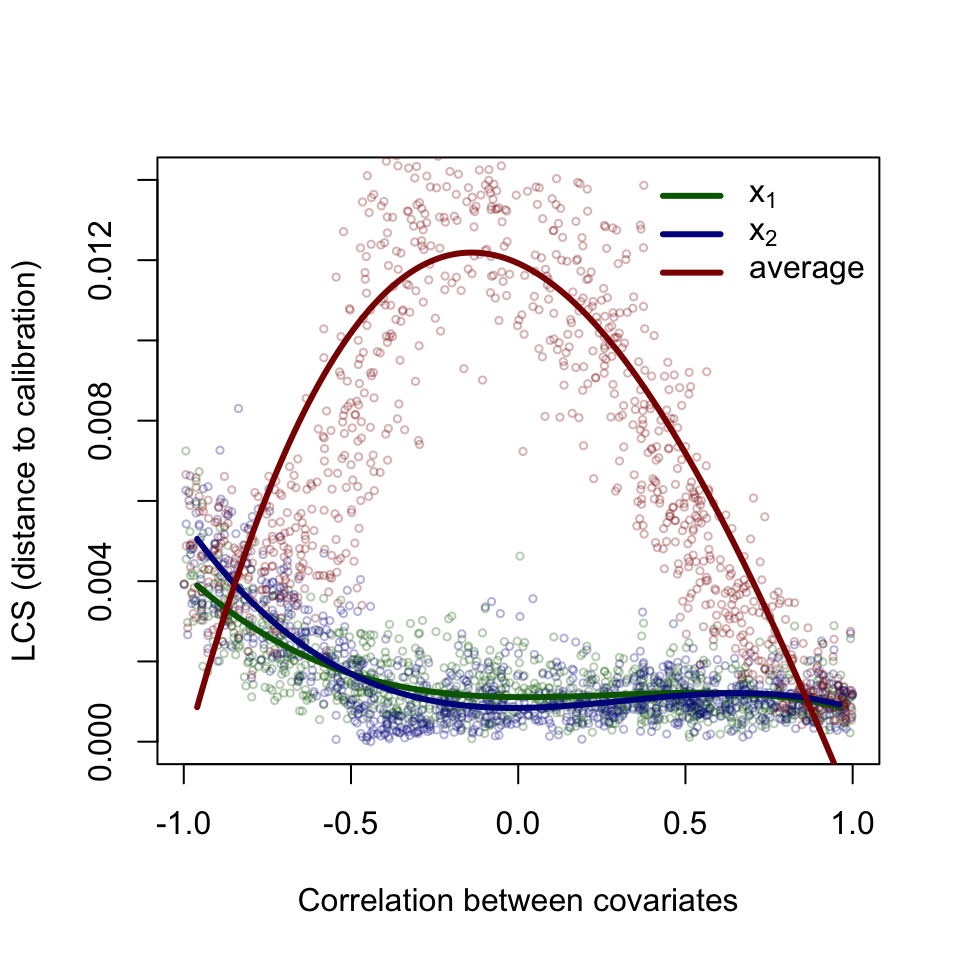}
    \caption{\textbf{Left:} learned monotone calibration maps $\widehat g(\cdot)$ for the two base models ($\widehat s_1$ in green, $\widehat s_2$ in blue) and for their average $\widehat s_{\mathrm{ens}}$ (red), shown here for $\rho=0$. The grey diagonal corresponds to perfect calibration.
    \textbf{Right:} local calibration score as a function of $\rho$ for the two base models (green/blue) and their average (red), computed on an independent test split.}
    \label{fig:averaging}
\end{figure}

\paragraph{Decomposition.} Figure~\ref{fig:brier-reducible} reports only the {\em reducible} part of the Brier decomposition, i.e.\ reliability $\mathbb E[(S-C)^2]$ and grouping $\mathbb E[(C-Q)^2]$ (the irreducible term $\mathbb E[Q(1-Q)]$ is omitted for readability).
Across all three scores, post-hoc recalibration drives the reliability component close to zero, while the grouping component remains essentially unchanged.
This confirms the interpretation of recalibration as a projection at information level $\sigma(S)$: it can remove miscalibration but cannot recover information lost in the compression $\bX\mapsto S$.
Moreover, using a richer score ($x12$) substantially reduces grouping, whereas quantization increases it, showing that calibration alone does not prevent large information loss.

\begin{figure}[!htb]
\centering
\includegraphics[width=\linewidth]{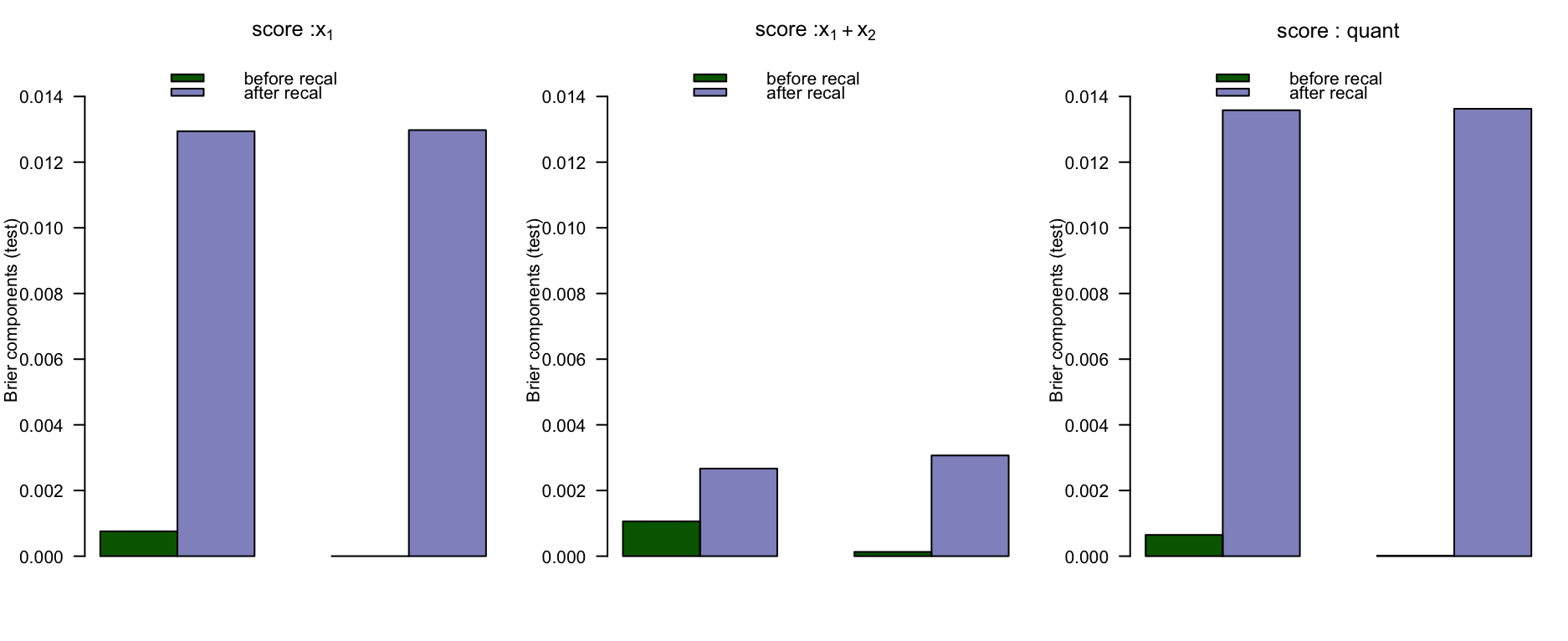}
\caption{\textbf{Reducible Brier components on test data.}
Bars show reliability $\widehat{\E}[(S-\widehat C)^2]$ and grouping $\widehat{\E}[(\widehat C-Q)^2]$ before and after monotone recalibration, for three scores: a logistic model on $X_1$ (left), a logistic model on $(X_1,X_2)$ (middle), and a quantized score (right).
Recalibration eliminates reliability but leaves grouping essentially unchanged.}
\label{fig:brier-reducible}
\end{figure}

\subsection{Real-data case study: GermanCredit}
\label{sec:experiments-real}

We complement the synthetic experiments with a real-data case study on \textsc{GermanCredit} \citep{hofmann1994statlog_german_credit}.
We consider two probabilistic classifiers trained on an independent training split:
(i) a logistic regression ({GLM}) and (ii) a random forest ({RF}).
We then form two ensembles: (a) {\em averaging} (arithmetic mean of predicted probabilities) and (b) {\em stacking} (a meta-logistic model trained on the calibration split using the pair of base scores).
Finally, we fit a {\em monotone} post-hoc recalibration map on an independent calibration split and evaluate metrics on a held-out test split.

Figure~\ref{fig:germancredit:reliability} displays monotone-smoothed calibration curves (raw in green, recalibrated in red; the diagonal corresponds to perfect calibration).
While recalibration substantially improves the {GLM} curve, nonparametric monotone recalibration may introduce distortions for some models (step-like behavior for tree-based scores), which can translate into degraded proper losses on the test set.
This finite-sample phenomenon is consistent with Appendix~\ref{app:calib-inference}: population improvements under strict propriety do not automatically carry over to plug-in estimators of $C$.

Table~\ref{tab:germancredit-gains} and Figure~\ref{fig:germancredit:summary} summarize test-set performance under two proper losses (Brier and log-loss), together with the miscalibration proxy LCS.
Proper-loss reliability estimates (for log-loss and Brier) and robustness analyses over repeated splits are reported in Appendix~\ref{app:dgp}.

\begin{table}[t]
\centering
\caption{\textsc{GermanCredit}: test-set performance before/after monotone recalibration.
Gains are defined as $\Delta=\text{raw}-\text{recal}$ (positive is better).}
\label{tab:germancredit-gains}
\footnotesize
\setlength{\tabcolsep}{4pt}
\begin{tabular}{lrrrrrrr}
\toprule
Model & \multicolumn{3}{c}{Brier} & \multicolumn{3}{c}{LogLoss} & LCS \\
\cmidrule(lr){2-4}\cmidrule(lr){5-7}
& raw & recal & $\Delta$ & raw & recal & $\Delta$ & \\
\midrule
GLM      & 0.19 & 0.18 & 0.01  & 0.73 & 0.55 & 0.19  & 0.012 \\
RF       & 0.17 & 0.17 & -0.00 & 0.52 & 0.63 & -0.11 & 0.014 \\
Average  & 0.17 & 0.18 & -0.01 & 0.52 & 0.75 & -0.24 & 0.013 \\
Stacking & 0.17 & 0.18 & -0.00 & 0.52 & 0.63 & -0.11 & 0.006 \\
\bottomrule
\end{tabular}
\end{table}

\begin{figure}[t]
\centering
\includegraphics[width=\linewidth]{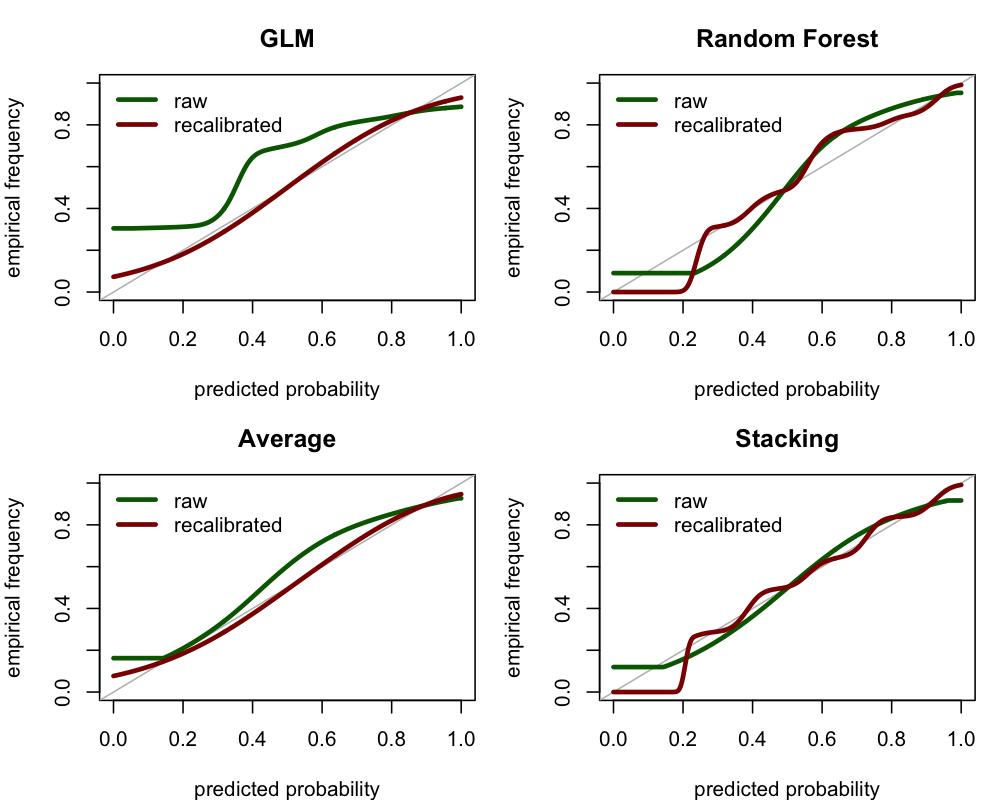}
\caption{\textsc{GermanCredit}: monotone-smoothed calibration curves on the test split.
Green: raw score; red: after monotone post-hoc recalibration fit on a separate calibration split.
The diagonal indicates perfect calibration.}
\label{fig:germancredit:reliability}
\end{figure}

\begin{figure}[t]
\centering
\includegraphics[width=\linewidth]{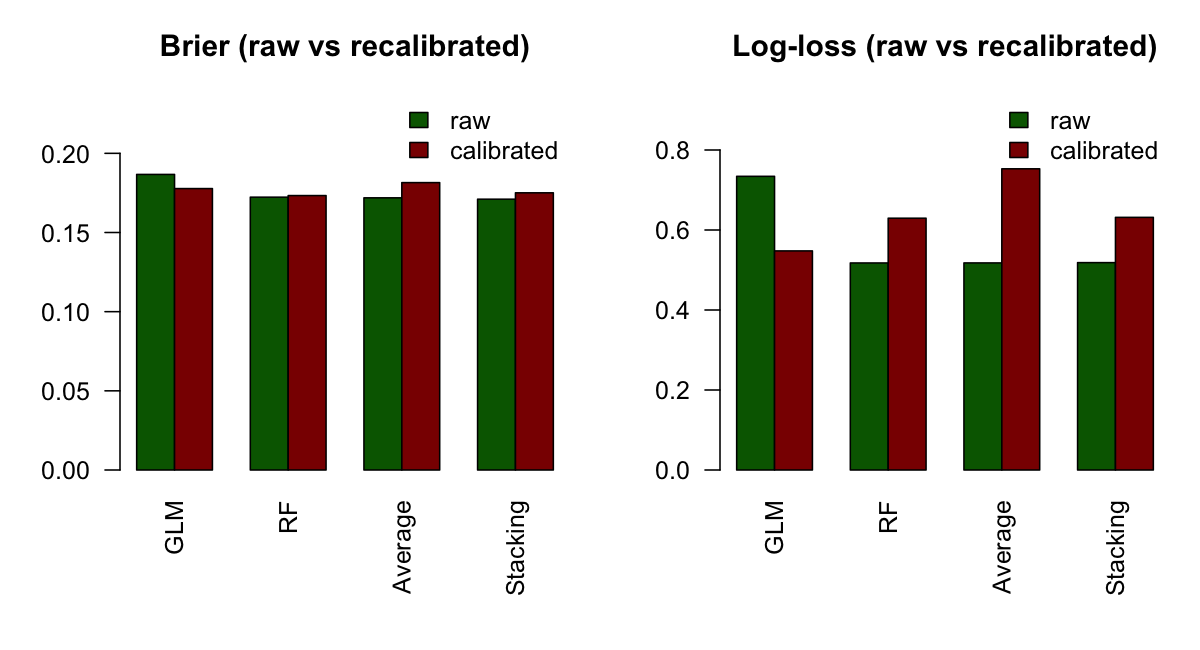}
\caption{\textsc{GermanCredit}: summary metrics on the test split.
Left: miscalibration proxy $\mathrm{LCS}=\widehat{\E}(S-\widehat C(S))^2$.
Middle/right: proper losses (Brier and log-loss) before vs.\ after recalibration.}
\label{fig:germancredit:summary}
\end{figure}


\section{Discussion and conclusion}
\label{sec:discussion}

This paper advocates an {\em information-level} view of calibration under proper losses.
Our main message is that calibration is inherently conditional: a predictor $S=s(\bX)$ is assessed through the $\sigma$-algebra it generates, and any compression $\bX\mapsto S$ induces an information loss that cannot be repaired by post-hoc recalibration.
The proper-loss decompositions developed here make this explicit by separating (i) a reliability term (miscalibration), (ii) a grouping term (loss of information from $\bX$ to $S$), and (iii) residual uncertainty at the feature level.
This separation clarifies what recalibration can and cannot improve, and explains why ensemble procedures may improve overall proper scores while degrading conditional calibration.

\paragraph{Implications for practice.}
Our results support a simple workflow: (1) diagnose performance using a proper-loss decomposition rather than a single scalar score, (2) use post-hoc recalibration to reduce the reliability component on held-out data, and (3) if grouping dominates, improve the score itself (features, model class, or ensembling that retains more information, e.g. stacking) rather than relying on calibration alone.
For sequential methods (boosting), the filtration-based viewpoint suggests interpreting stagewise improvements as refinement gains, while recognizing that reliability may drift and typically benefits from a final recalibration step.

\paragraph{Limitations.}
Our decomposition identities are population statements; empirical estimates require additional modeling choices, in particular when approximating conditional laws such as $C=\mathbb P(Y\in\cdot\mid S)$ (and, beyond simulations, $Q=\mathbb P(Y\in\cdot\mid \bX)$).
Moreover, calibration is only one desideratum: imposing monotonicity or smoothness stabilizes estimation and can preserve ranking, but may also trade off resolution when the score is coarse or heavily discretized.

\paragraph{A geometric perspective: recalibration as transport on $[0,1]$ (and beyond).}
In the binary case, a strictly increasing recalibrator $g:[0,1]\to[0,1]$ can be interpreted as a transport map pushing the distribution of $S$ forward to that of the calibrated score $C=\mathbb P(Y=1\mid S)$.
This suggests a geometric lens on recalibration: compare calibrators by the size of the transport they induce (e.g. Wasserstein distances between the distributions of $S$ and $g(S)$), or view regularized calibration as selecting a map that balances fidelity to empirical calibration with minimal distortion of the score distribution.
Extending this viewpoint to multi-class calibration would naturally involve geometry on the simplex $\Delta(\mathcal Y)$ and raises questions about appropriate transport/metric structures for probability vectors.

\paragraph{Outlook.}
Several directions appear promising.
First, the information-level framework invites {\em calibration-aware ensembling} principles: characterize aggregators that preserve calibration (or preserve it after a controlled projection), and quantify the price paid in grouping when compressing a rich ensemble output to a single score.
Second, developing statistical inference tools for decomposition terms could turn these identities into routine model diagnostics, particularly under distribution shift.
Finally, extending the geometric transport viewpoint to multi-class settings, and relating it to proper-loss optimization and monotone/smooth constraints, may yield new calibration methods with explicit control of both reliability and distortion.

\paragraph{Conclusion.}
Calibration is not an absolute property of a model output but a conditional statement at a chosen information level.
Proper-loss decompositions provide a unified language to separate reliability, information loss, and residual uncertainty, and to reason about recalibration, aggregation, and sequential ensemble procedures in a common framework.

\bibliography{biblio}

\clearpage

\onecolumn

\title{\titredupaper\\(Supplementary Material)}
\maketitle

\appendix

\section{Proofs}

\subsection{Proof of Theorem~\ref{thm:onelevel}}

\begin{proof}
By the tower property of conditional expectation and the definition of $L$,
\[
  \mathbb E[\ell(T,Y)]
  =\mathbb E\Big[\mathbb E[\ell(T,Y)\mid \mathcal A]\Big]
  =\mathbb E\Big[L(T,Q_{\mathcal A})\Big].
  \]
Adding and subtracting $\mathcal E_\ell(Q_{\mathcal A})$ yields
\begin{align*}
\mathbb E[L(T,Q_{\mathcal A})]
&=\mathbb E\big[L(T,Q_{\mathcal A})-\mathcal E_\ell(Q_{\mathcal A})\big]+
  \mathbb E[\mathcal E_\ell(Q_{\mathcal A})\notag
            =\mathbb E[d_\ell(T,Q_{\mathcal A})]+
              \mathbb E[\mathcal E_\ell(Q_{\mathcal A})]. 
            \end{align*}
            \end{proof}
            
\subsection{Proof of Theorem~\ref{thm:chain}}

\begin{proof}
            First apply Theorem~\ref{thm:onelevel} with $\mathcal A$:
              \[
                \mathbb E[\ell(T,Y)] = \mathbb E[d_\ell(T,Q_{\mathcal A})]+\mathbb E[\mathcal E_\ell(Q_{\mathcal A})].
                \]
            It remains to decompose $\mathbb E[\mathcal E_\ell(Q_{\mathcal A})]$ through $\mathcal B$.
            Note that $Q_{\mathcal A}$ is $\mathcal A$-measurable and hence also $\mathcal B$-measurable since $\mathcal A\subseteq\mathcal B$.
            Apply Theorem~\ref{thm:onelevel} with $T:=Q_{\mathcal A}$ and $\mathcal B$:
              \[
                \mathbb E[\ell(Q_{\mathcal A},Y)]
                =
                  \mathbb E\big[d_\ell(Q_{\mathcal A},Q_{\mathcal B})\big]+\mathbb E\big[\mathcal E_\ell(Q_{\mathcal B})\big].
                \]
            Moreover,
            \[
              \mathbb E[\ell(Q_{\mathcal A},Y)]
              =
                \mathbb E\big[\mathbb E[\ell(Q_{\mathcal A},Y)\mid \mathcal A]\big]
              =
                \mathbb E\big[\mathcal E_\ell(Q_{\mathcal A})\big],
              \]
            by the definition of $\mathcal E_\ell$.
            Combining these identities yields the claimed equality.
\end{proof}
           
\subsection{Proof of Lemma~\ref{rem:grouping-zero}}

\begin{proof}
Since $d_\ell(\cdot,\cdot)\ge 0$, the condition $\mathbb E[d_\ell(C,Q)]=0$ implies $d_\ell(C,Q)=0$ a.s.
By strict propriety, $d_\ell(C,Q)=0$ a.s.\ iff $C=Q$ a.s.
Finally, $C=\mathbb P(Y\in\cdot\mid S)$ and $Q=\mathbb P(Y\in\cdot\mid \bX)$ satisfy $C=Q$ a.s.\ iff $Q$ admits a $\sigma(S)$-measurable version, which is exactly $\sigma(S)$-measurability of $Q$.
\end{proof}

\subsection{Proof of Proposition~\ref{prop:urc}}

\begin{proof}[Proof]
Start from the one-level decomposition with $\mathcal A=\mathcal{H}_S$:
$\mathbb{E}[\ell(S,Y)]=\mathbb{E}[d_\ell(S,C)]+\mathbb{E}[\mathcal{E}_{\ell}(C)]$.
It remains to express $\mathbb{E}[\mathcal{E}_\ell(C)]$ relative to the constant $\bar Q$.
Applying Theorem~\ref{thm:onelevel} with $T=\bar Q$ (constant, hence $\mathcal{H}_S$-measurable) and $\mathcal A=\mathcal{H}_S$ yields $\mathbb{E}[\ell(\bar Q,Y)]=\mathbb{E}[d_\ell(\bar Q,C)]+\mathbb{E}[\mathcal{E}_\ell(C)]$.
But $\mathbb{E}[\ell(\bar Q,Y)]=\mathcal{E}_\ell(\bar Q)$, hence $\mathbb{E}[\mathcal{E}_\ell(C)]=\mathcal{E}_\ell(\bar Q)-\mathbb{E}[d_\ell(\bar Q,C)]$ and the claimed formula follows.
A complete treatment (including ``sufficiency'' attributes) is developed in~\citet{broecker2018reliability}.
\end{proof}

\subsection{Proof of Theorem~\ref{thm:four}}

\begin{proof}
First apply Theorem~\ref{thm:chain} to the chain $\mathcal{H}_S\subseteq\mathcal{H}_{\bX}$ to obtain~\eqref{eq:clf-chain}.
Next apply Theorem~\ref{thm:onelevel} with $T=Q$ and $\mathcal A=\mathcal{H}_{Z}$ to get $\mathbb{E}[\ell(Q,Y)]=\mathbb{E}[d_\ell(Q,\Pi)]+\mathbb{E}[\mathcal{E}_\ell(\Pi)]$.
Moreover, $\mathbb{E}[\ell(Q,Y)]=\mathbb{E}[\mathcal{E}_\ell(Q)]$ by the same argument as in the proof of Theorem~\ref{thm:chain}.
Substituting $\mathbb{E}[\mathcal{E}_\ell(Q)]$ into~\eqref{eq:clf-chain} yields the result.
\end{proof}

\subsection{Proof of Corollary~\ref{cor:brier}}

\begin{proof}
For the Brier loss one checks directly that $L(p,q)=(p-q)^2+q(1-q)$, hence $d_\ell(p,q)=(p-q)^2$ and $\mathcal{E}_\ell(q)=q(1-q)$.
Applying Theorem~\ref{thm:chain} to $\mathcal{H}_S\subseteq\mathcal{H}_{\bX}$ yields the first equality.
The second follows from Theorem~\ref{thm:onelevel} with $T=Q$ and $\mathcal A=\mathcal{H}_{Z}$, noting that
$d_\ell(Q,\Pi)=(Q-\Pi)^2$ and $\mathcal{E}_\ell(\Pi)=\Pi(1-\Pi)$.
\end{proof}

\subsection{Proof of Corollary~\ref{cor:logloss-info}}

\begin{proof}
For log-loss, $L(p,q)=H(q)+\mathrm{KL}(q\|p)$, hence
$d_{\log}(p,q)=\mathrm{KL}(q\|p)$ and $\mathrm{E}_{\log}(q)=H(q)$.
Plugging into Theorem~\ref{thm:chain} yields the first display. For the second, note that
$I(Y;\mathcal B\mid \mathcal A)=\mathbb E[\mathrm{KL}(P(Y\in\cdot\mid \mathcal B)\|
                                                        P(Y\in\cdot\mid \mathcal A))]=\mathbb E[\mathrm{KL}(Q_{\mathcal B}\|Q_{\mathcal A})]$.
Finally, when $S=s(X)$, $P(Y\mid X,S)=P(Y\mid X)$, so
$\mathbb E[\mathrm{KL}(Q\|C)]=I(Y;X\mid S)$ and $H(Y\mid X)=\mathbb E[H(Q)]$.
\end{proof}

\subsection{Proof of Proposition~\ref{prop:calib-regret}}

\begin{proof}
The decomposition is Theorem~\ref{thm:onelevel} with $\mathcal A=\sigma(S)$ and $T=S$.
Strict propriety yields $\mathbb E[d_\ell(S,C)]=0 \Rightarrow S=C$ a.s.
Finally, if $S=C$ a.s., then $\mathbb E[\onehot(Y)\mid S]=S$ a.s., hence $\mathbb E[\onehot(Y)]=\mathbb E[S]$ by the tower property.
\end{proof}

\subsection{Proof of Proposition~\ref{prop:pop-opt-recal}}

\begin{proof}
Apply Theorem~\ref{thm:onelevel} with $\mathcal{A}=\sigma(S)$ and $T=g(S)$:
  \[
    \mathbb{E}[\ell(g(S),Y)]
    =\mathbb{E}\big[d_\ell(g(S),C)\big]+\mathbb{E}\big[\mathcal E_\ell(C)\big].
    \]
The second term does not depend on $g$, while the first is minimized at $0$, achieved by $g(S)=C$ a.s. (and uniquely if $\ell$ is strictly proper).
Finally, $\mathbb{E}[\ell(C,Y)]=\mathbb{E}[\mathcal E_\ell(C)]$ since $d_\ell(C,C)=0$.
\end{proof}

\subsection{Proof of Proposition~\ref{prop:monotone-preserves-roc}}

\begin{proof}
For fixed $\tau$, the set $\{s:g(s)\ge\tau\}$ is an interval $[t,1]$ since $g$ is non-decreasing.
Let $\tau':=\inf\{s:g(s)\ge\tau\}$. Then $g(S)\ge\tau$ iff $S\ge\tau'$ except possibly on flat regions (ties), which only affect ROC through the chosen tie-breaking convention.
\end{proof}

\subsection{Proof of Theorem~\ref{thm:boosting-telescope}}

\begin{proof}
Apply Theorem~\ref{thm:chain} to $(\mathcal F_0,\mathcal F_1)$ and predictor $T$:
  \[
    \mathbb E[\ell(T,Y)]
    =
      \mathbb E[d_\ell(T,Q_0)]
    +\mathbb E[d_\ell(Q_0,Q_1)]
    +\mathbb E[\mathcal E_\ell(Q_1)].
    \]
Then apply the same identity to $\mathbb E[\mathcal E_\ell(Q_1)]=\mathbb E[\ell(Q_1,Y)]$ along $(\mathcal F_1,\mathcal F_2)$, and iterate.
The divergence terms telescope, yielding the claim.
\end{proof}

\subsection{Proof of Corollary~\ref{cor:filtration-logloss-info}}

\begin{proof}
Immediate from Corollary~\ref{cor:logloss-info} applied to each pair
$\mathcal F_t\subseteq \mathcal F_{t+1}$ and the definition of conditional mutual information.
\end{proof}

\subsection{Proof of Corollary~\ref{cor:boosting-brier}}

\begin{proof}
Let $\eta_t=\mathbb E[Y\mid\mathcal F_t]$. Write $Y-\eta_t=(Y-\eta_{t+1})+(\eta_{t+1}-\eta_t)$ and expand the square.
The cross term vanishes since $\mathbb E[Y-\eta_{t+1}\mid\mathcal F_{t+1}]=0$.
\end{proof}



\section{Conditional laws as $\Delta(\mathcal Y)$-valued random variables}
\label{app:condlaw}

Throughout, $\mathcal Y=\{1,\dots,k\}$ is finite and $\Delta(\mathcal Y)$ denotes the probability simplex.
For any $\sigma$-algebra $\mathcal A\subseteq\mathcal F$, the conditional law
\[
Q_{\mathcal A}:=\mathbb P(Y\in\cdot\mid \mathcal A)
\]
can be represented as a $\Delta(\mathcal Y)$-valued random vector.

When $\mathcal Y$ is finite, the conditional law $\mathbb P(Y\in\cdot\mid\mathcal A)$ can be identified with the $\mathcal A$-measurable vector of conditional probabilities $\big(\mathbb P(Y=1\mid\mathcal A),\dots,\mathbb P(Y=k\mid\mathcal A)\big)$.
More generally, one may view $\omega\mapsto Q_{\mathcal A}(\omega)$ as an $\mathcal A$-measurable {\em Markov kernel} from $(\Omega,\mathcal A)$ to $(\mathcal Y,2^{\mathcal Y})$.
All equalities involving $Q_{\mathcal A}$ are understood up to $\mathbb P$-almost sure equality (versions).

\begin{lemma}[Vector representation]
\label{lem:vector-QA}
Let $\onehot(Y)\in\{0,1\}^k$ be the one-hot encoding of $Y$.
Define the $\mathcal A$-measurable random vector
\[
\pi_{\mathcal A}:=\mathbb E[\onehot(Y)\mid \mathcal A]\in\R^k.
\]
Then $\pi_{\mathcal A}\in\Delta(\mathcal Y)$ almost surely and, for each $j\in\{1,\dots,k\}$,
\[
(\pi_{\mathcal A})_j=\mathbb P(Y=j\mid \mathcal A)\quad\text{a.s.}
\]
In particular, identifying a distribution $q\in\Delta(\mathcal Y)$ with its coordinate vector $q=(q_1,\dots,q_k)$, we may view $Q_{\mathcal A}$ as the $\Delta(\mathcal Y)$-valued random variable $Q_{\mathcal A}\equiv \pi_{\mathcal A}$.
\end{lemma}

\begin{proof}
For each $j$, the coordinate $(\pi_{\mathcal A})_j=\mathbb E[\mathbf 1_{\{Y=j\}}\mid\mathcal A]$ is $\mathcal A$-measurable and lies in $[0,1]$ almost surely.
Moreover, summing over $j$ gives
$$
\sum_{j=1}^k (\pi_{\mathcal A})_j=\mathbb E[\sum_{j=1}^k \mathbf 1_{\{Y=j\}}\mid\mathcal A]=1,~\text{a.s.},
$$
hence $\pi_{\mathcal A}\in\Delta(\mathcal Y)$ a.s.
The identity $(\pi_{\mathcal A})_j=\mathbb P(Y=j\mid\mathcal A)$ is the defining property of conditional probabilities.
\end{proof}

\begin{lemma}[Uniqueness up to a.s.\ equality]
\label{lem:unique-version}
Let $U$ and $V$ be $\Delta(\mathcal Y)$-valued $\mathcal A$-measurable random variables such that for every $B\subseteq\mathcal Y$,
\[
U(B)=\mathbb P(Y\in B\mid \mathcal A)\quad\text{a.s.}
\qquad\text{and}\qquad
V(B)=\mathbb P(Y\in B\mid \mathcal A)\quad\text{a.s.}
\]
Then $U=V$ almost surely.
In particular, $Q_{\mathcal A}$ is well-defined as an $\mathcal A$-measurable $\Delta(\mathcal Y)$-valued random variable up to almost sure equality.
\end{lemma}

\begin{proof}
Because $\mathcal Y$ is finite, it suffices to check equality on singletons. For each $j\in\{1,\dots,k\}$ we have $U(\{j\})=V(\{j\})$ a.s., hence the coordinate vectors coincide a.s., so $U=V$ a.s.
\end{proof}

Finally, note that for any $\mathcal A\subseteq\mathcal B\subseteq\mathcal F$, the corresponding conditional laws satisfy $Q_{\mathcal A}=\mathbb E[Q_{\mathcal B}\mid \mathcal A]$ coordinate-wise, i.e.\ by applying the tower property to each indicator $\mathbf 1_{\{Y=j\}}$.
\section{Statistical uncertainty for calibration components}
\label{app:calib-inference}

Our decompositions are population identities, but empirical work requires estimating quantities such as $\mathbb E[d_\ell(S,C)]$, $\mathbb E[d_\ell(C,Q)]$, and $\mathbb E[\mathcal E_\ell(Q)]$ on finite samples.
A practical question is therefore: how uncertain are these estimates under sampling variability?

\paragraph{Avoiding optimistic bias.}
Because $C=\mathbb P(Y\in\cdot\mid S)$ is typically estimated from data, evaluating calibration on the same sample used to fit the calibrator can lead to optimistic (downward biased) reliability estimates.
We therefore recommend a train/calibration/test split, or cross-fitting, so that both the predictor and the calibration map are evaluated out-of-sample.

\paragraph{Bootstrap inference (pipeline view).}
A simple and broadly applicable approach is resampling-based inference (bootstrap) applied to the entire pipeline: refit the predictor on resampled training data, refit the calibrator (or calibration curve) on resampled calibration data, and evaluate the decomposition terms on held-out test data.
Repeating this procedure yields empirical confidence intervals for reliability, grouping, and total expected loss.

\paragraph{Two sources of uncertainty.}
Depending on the scientific question, one may distinguish:
(i) {\em calibration-only uncertainty}, where the base predictor is treated as fixed and only the calibrator is refit (resampling the calibration split), and
(ii) {\em end-to-end uncertainty}, where both the predictor and calibrator are refit (resampling train and calibration splits).
The latter is typically wider but better reflects deployment variability.

\paragraph{Parametric calibrators.}
For smooth parametric calibrators (e.g.\ Platt scaling or temperature scaling), asymptotic normality can provide standard errors for low-dimensional summaries (e.g.\ integrated calibration error), via standard M-estimation theory.
In practice, bootstrap procedures remain robust and easy to implement across model classes and calibration estimators.

\section{Additional experiment details}
\label{app:exp-synth}\label{app:dgp}

\subsection{Synthetic experiment details}\label{app:D:synthetic}

We generate binary outcomes $Y\in\{0,1\}$ from a controlled two-feature model where the true conditional probability
$Q(\bX)=\mathbb P(Y=1\mid \bX)$ is known in closed form.
Let $(Z_1,Z_2)$ be jointly Gaussian with correlation $\rho\in(-1,1)$ and define
\[
X_1=\Phi(Z_1),\qquad X_2=\Phi(Z_2),
\]
so that $X_1,X_2\in(0,1)$ have uniform marginals and dependence controlled by $\rho$.
Let $\sigma(u)=\{1+\exp(-u)\}^{-1}$ denote the logistic sigmoid and set
\[
\eta(\bX)
:=2.5\big((X_1+X_2)-1\big)+2\big(\exp((X_1-X_2)^3)-1\big),
\qquad
Q(\bX)=\sigma\big(\eta(\bX)\big).
\]
We then sample
\[
Y\mid \bX \sim \mathrm{Bernoulli}\big(Q(\bX)\big).
\]
We vary $\rho$ to control redundancy between $X_1$ and $X_2$.
Unless stated otherwise, we use independent train/calibration/test splits with $n$ samples each.

\begin{figure}[t]
\centering
\includegraphics[width=0.245\linewidth]{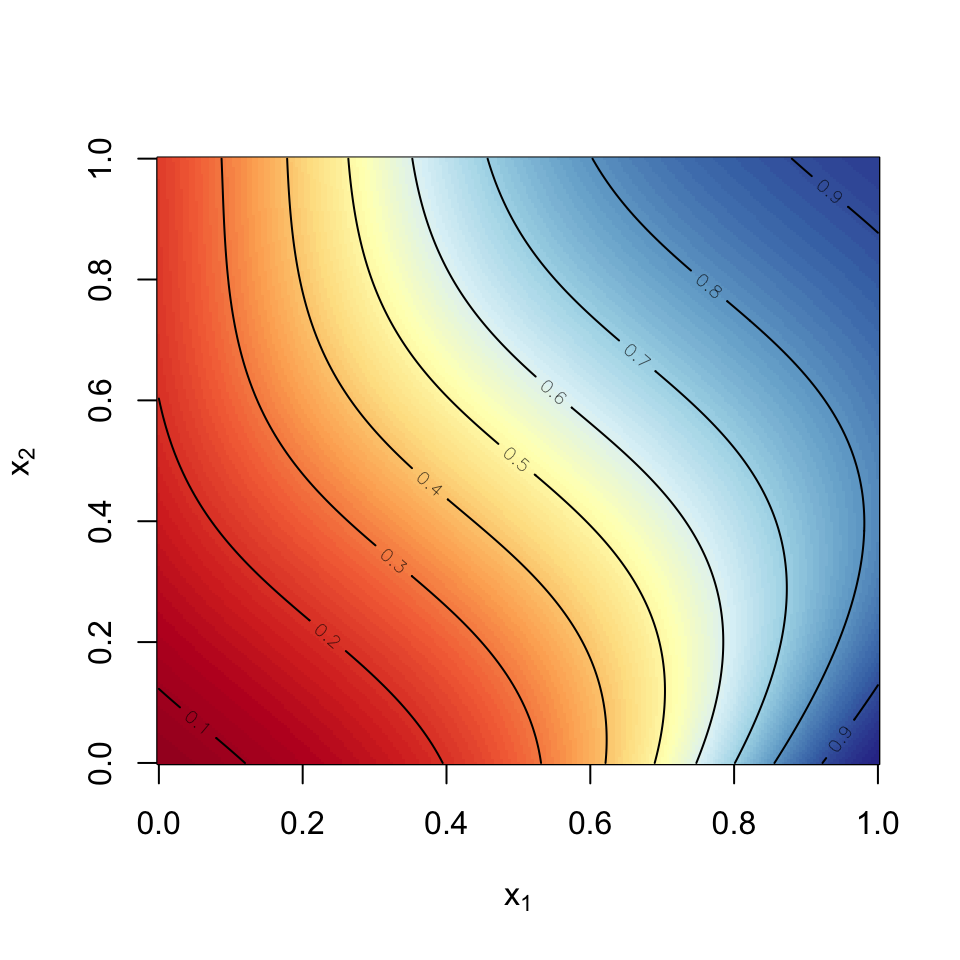}\includegraphics[width=0.245\linewidth]{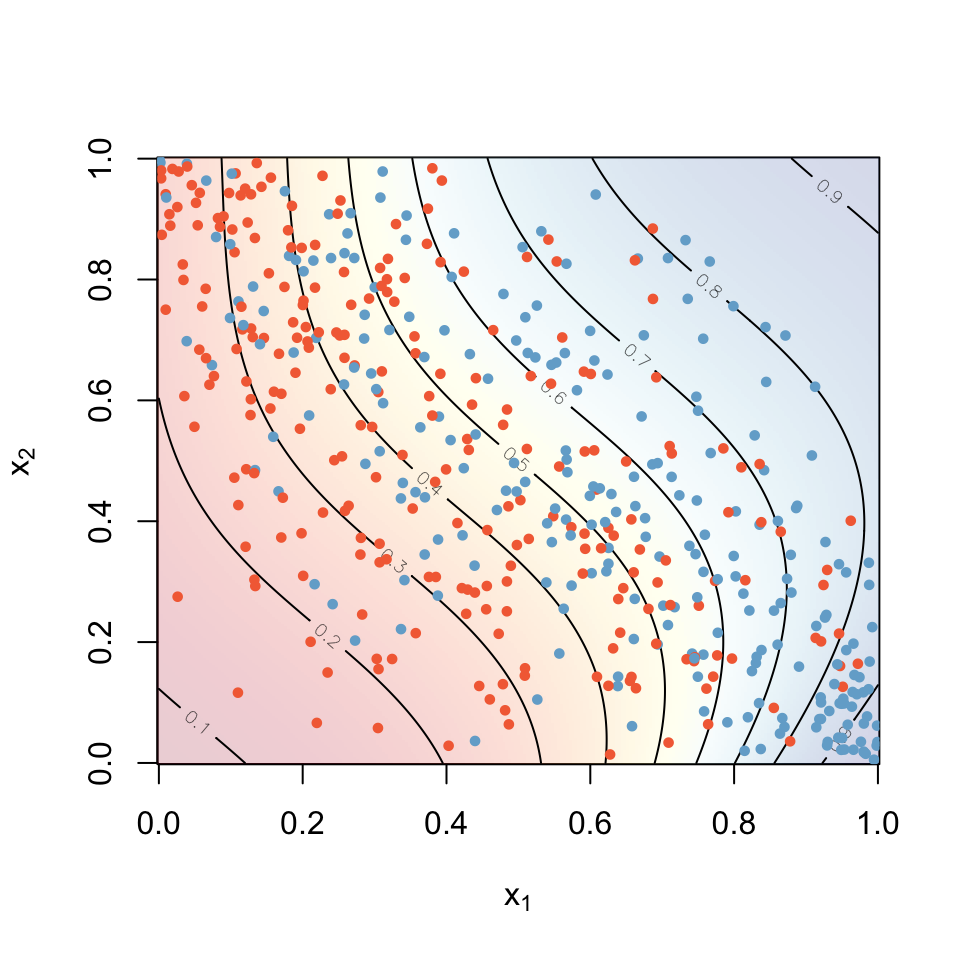}\includegraphics[width=0.245\linewidth]{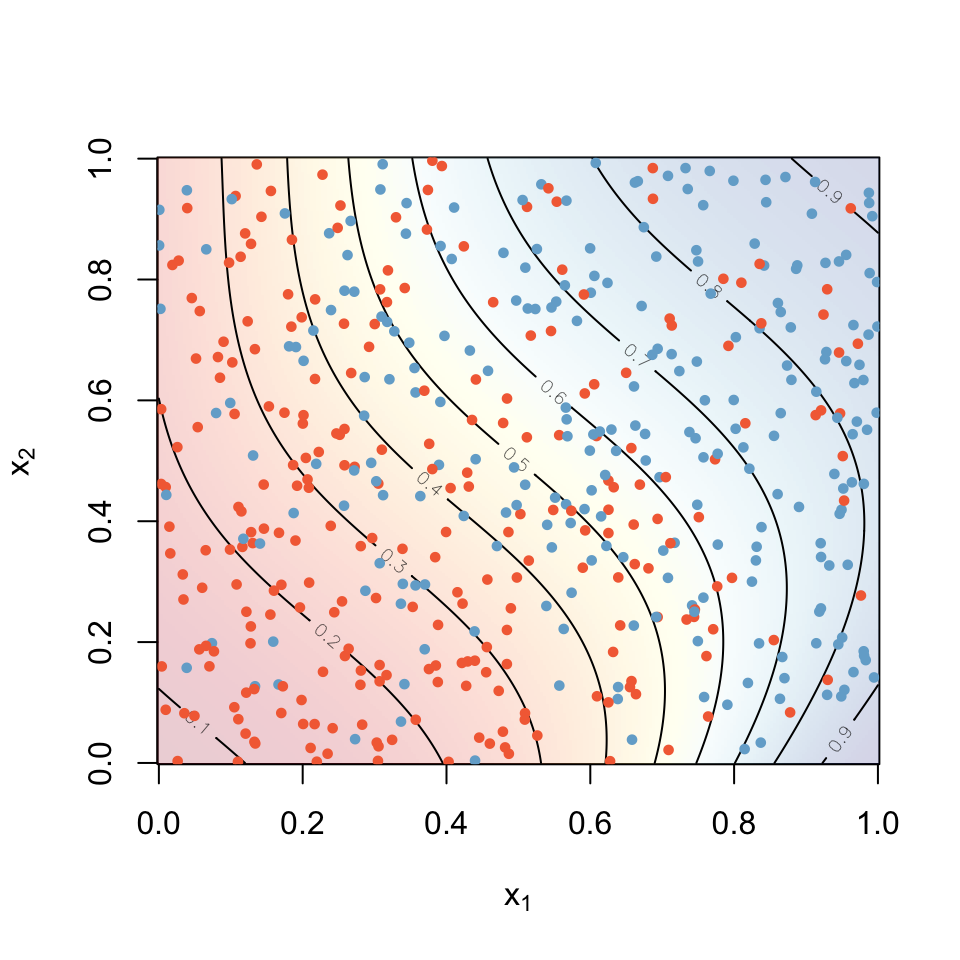}\includegraphics[width=0.245\linewidth]{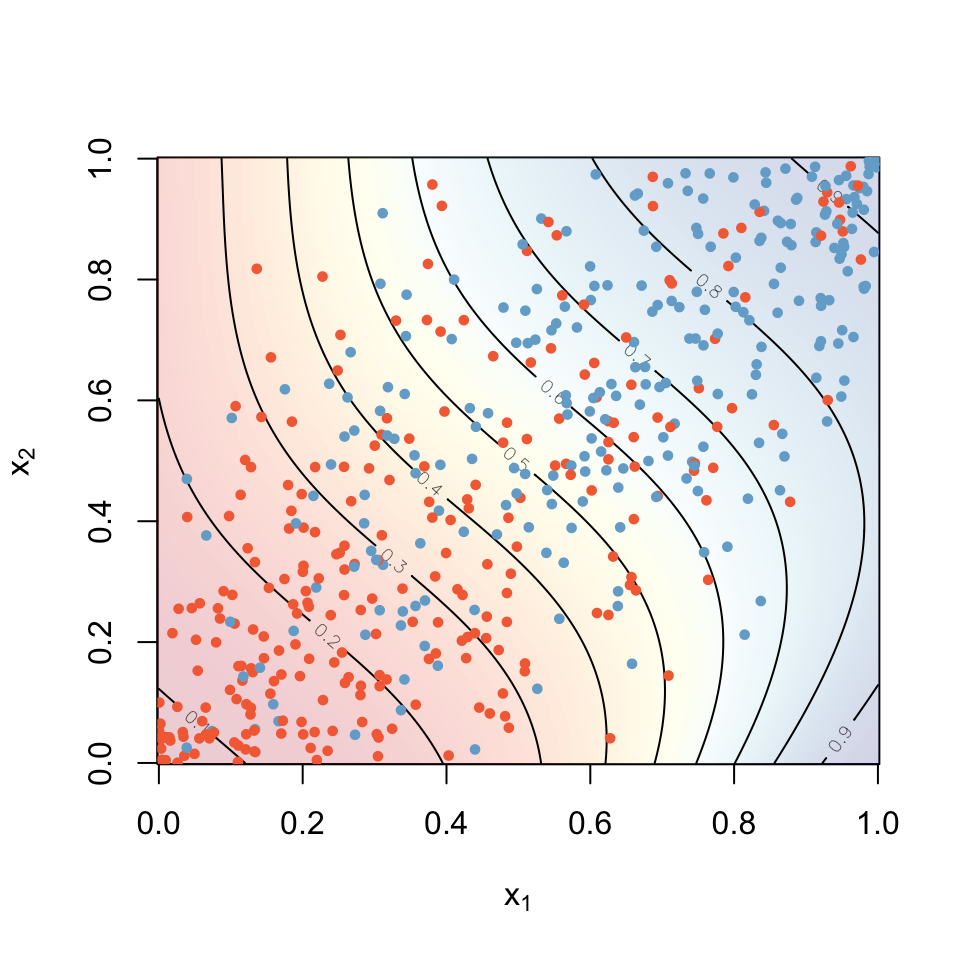}
\caption{True conditional probability surface $Q(x_1,x_2)$ for the synthetic model, on the left hand side. Then 3 scatterplots of simulated data, with $\rho=-0.7$ (middle left), $\rho=0$ (middle right), and $\rho=0.7$ (right).}
\label{fig:true-surface}
\end{figure}

\paragraph{Splits.}
Each run uses independent train/calibration/test splits.
Calibration mappings are fit {\em only} on the calibration split to avoid optimistic bias.

\paragraph{Estimating $C(s)=\mathbb P(Y=1\mid S=s)$.}
In the binary case, we fit monotone smooth calibrators of the form $\widehat C=\widehat g(S)$.
We use (i) Platt scaling, (ii) isotonic regression, and (iii) a monotone spline estimator implemented via constrained GAM/P-splines (SCAM), optionally preceded by kernel pre-smoothing on a grid.
Unless stated otherwise, all reported numbers use isotonic regression as the monotone calibrator; Platt scaling and SCAM yield similar qualitative conclusions. Methodological details are provided in Appendix~\ref{app:g:smooth}. An anonymized implementation will be released upon publication.

\paragraph{Estimating decomposition terms.}
Because $Q(\bX)$ is known in the simulation, we evaluate $H_\ell(Q)$ exactly (plug-in), and compute the grouping term $\mathbb E[d_\ell(C,Q)]$ by plugging in $\widehat C(S)$ and the known $Q(\bX)$ on the test set.
Reliability is computed as $\mathbb E[d_\ell(S,\widehat C(S))]$.
For the Brier score, these become squared differences; for log-loss, they become KL divergences.

\subsection{Additional diagnostics and robustness}\label{app:D:robust}

\subsubsection*{LCS proxy and reliability diagrams}\label{app:D:lcs}
For completeness we report an $\ell_2$-style proxy for calibration error,
\[
\mathrm{LCS} \;=\; \widehat{\E}\big[(S-\widehat C(S))^2\big],
\]
computed by fitting a calibrator $\widehat C$ on a held-out fold (or via cross-fitting) and evaluating on the corresponding test fold.
We emphasize that our theoretical decompositions are in terms of proper-loss reliability
$\widehat{\mathrm{Rel}}_\ell=\frac1n\sum_{i=1}^n d_\ell(S_i,\widehat C(S_i))$; LCS is reported only as a familiar $\ell_2$ proxy.

Reliability diagrams use $B$ quantile bins on $[0,1]$ (approximately equal-mass bins), plotting $\widehat{\E}[Y\mid S\in\text{bin}]$ against the bin-wise mean score.

\subsubsection*{Uncertainty via repeated splits}\label{app:D:splits}
To assess variability, we repeat the full pipeline over $R$ random train/validation/test splits (or seeds).

Table~\ref{tab:robustness} reports mean$\pm$sd over $R=100$ randomized splits for raw/recalibrated losses and proper-loss reliability.
These averages summarize end-to-end variability of the full pipeline (train $\to$ calibrate $\to$ test).
Here, {\em averaging} denotes the arithmetic mean of base probabilities, while {\em stacking} denotes a meta-logistic model trained on the calibration split from the pair of base scores.
We report mean$\pm$std for log-loss, Brier, and proper-loss reliability.

\begin{table}

\caption{\label{tab:robustness}
{\sc GermanCredit} robustness over repeated splits (mean$\pm$sd over $R=100$ seeds). Proper-loss reliability uses the fitted calibrator $\widehat C$ trained on the calibration split and evaluated on the test split.}
\centering
\begin{tabular}[t]{lccc}
\toprule
Method & Log-loss (raw) & Log-loss (recal.) & Rel$_{\log}$\\
\midrule
Average & 0.496 $\pm$ 0.031 & 0.601 $\pm$ 0.093 & 0.028 $\pm$ 0.014\\
GLM & 0.565 $\pm$ 0.084 & 0.640 $\pm$ 0.124 & 0.033 $\pm$ 0.026\\
RF & 0.501 $\pm$ 0.027 & 0.549 $\pm$ 0.071 & 0.040 $\pm$ 0.014\\
Stacking & 0.504 $\pm$ 0.037 & 0.589 $\pm$ 0.085 & 0.021 $\pm$ 0.008\\
\bottomrule
\end{tabular}
\begin{tabular}[t]{lccc}
\toprule
Method & Brier (raw) & Brier (recal.) & Rel$_{\mathrm{Brier}}$\\
\midrule
Average & 0.165 $\pm$ 0.013 & 0.171 $\pm$ 0.012 & 0.007 $\pm$ 0.003\\
GLM & 0.176 $\pm$ 0.016 & 0.177 $\pm$ 0.014 & 0.008 $\pm$ 0.004\\
RF & 0.166 $\pm$ 0.012 & 0.168 $\pm$ 0.015 & 0.012 $\pm$ 0.005\\
Stacking & 0.167 $\pm$ 0.015 & 0.171 $\pm$ 0.015 & 0.004 $\pm$ 0.002\\
\bottomrule
\end{tabular}
\end{table}

Table~\ref{tab:germancredit-tests} reports a paired, split-wise comparison against the {\em Average} ensemble over $R$ randomized train/calibration/test splits.
For each split we compute the recalibrated log-loss and the proper log-loss reliability term $\widehat{\mathrm{Rel}}_{\log}$, and report $\Delta =$ (method $-$ Average), so that negative values indicate an improvement over the reference.
In addition to mean $\pm$ sd of $\Delta$, we report the win-rate (fraction of splits where the method outperforms Average) and a one-sided Wilcoxon signed-rank test of $H_1$: method $<$ Average, with Holm correction across methods.
Overall, the results suggest that stacking yields a statistically significant reduction in $\widehat{\mathrm{Rel}}_{\log}$ (64\% win-rate, corrected $p<10^{-4}$), while differences in recalibrated log-loss are small and not significant after correction.

\begin{table}
\caption{Paired comparison against {\em Average} (computed per split/seed). $\Delta$ denotes (method $-$ reference) so negative values indicate an improvement over the reference. We report win-rate and one-sided Wilcoxon signed-rank tests ($H_1$: method $<$ reference), with Holm correction across methods. \label{tab:germancredit-tests}}
\centering
\begin{tabular}[t]{lcccccc}
\toprule
Method & $\Delta$ Log-loss (recal.) & Win\% & $p$ (Holm) & $\Delta$ Rel$_{\log}$ & Win\%  & $p$ (Holm) \\
\midrule
GLM & 0.0281 $\pm$ 0.1091 & 34\% & 0.996 & 0.0141 $\pm$ 0.0388 & 39\% & 1\\
RF & -0.0184 $\pm$ 0.0963 & 60\% & 0.056 & 0.0134 $\pm$ 0.0140 & 15\% & 1\\
Stacking & 0.0020 $\pm$ 0.0753 & 55\% & 0.949 & -0.0055 $\pm$ 0.0109 & 64\% & $<10^{-4}$\\
\bottomrule
\end{tabular}
\end{table}

\subsubsection*{Protocol and implementation details}\label{app:D:details}
Unless stated otherwise, calibration is estimated out-of-sample, either using an explicit calibration split (as in Section~\ref{sec:experiments-real}) or via K-fold cross-fitting, preventing optimistic bias.
Hyperparameters are selected on validation folds only; test folds are used once per split for reporting.


\section{Shape-constrained smoothing for (re)calibration}\label{app:g:smooth}

We observe pairs $(x_i,y_i)_{i=1}^n$ with $x_i\in[0,1]$ (uncalibrated scores/probabilities) and $y_i\in\{0,1\}$ (or more generally $y_i\in[0,1]$). We seek a recalibration map $g:[0,1]\to[0,1]$, typically required to be {\em nondecreasing}.
For a discussion of interpretability and order-consistency desiderata (including why monotone $g$ preserves ROC/AUC), see Section~\ref{sec:recalibration} and Proposition~\ref{prop:monotone-preserves-roc}.
We focus here on estimation aspects under shape constraints.

\subsection{Isotonic regression (piecewise-constant monotone baseline)}\label{app:isotonic}
Let $x_{(1)}\le \cdots \le x_{(n)}$ denote the sorted inputs, with associated responses $y_{(i)}$.
The (weighted) isotonic regression estimator solves
\begin{equation}
\label{eq:isotonic}
\min_{\theta_1\le \cdots \le \theta_n}
\frac12\sum_{i=1}^n w_i(\theta_i-y_{(i)})^2,
\qquad w_i\ge 0.
\end{equation}
This yields fitted values $\hat\theta_i$ at the design points.
A function $g_{\mathrm{iso}}(x)$ is then obtained by a standard interpolation scheme, typically {\em piecewise constant} (right-continuous steps) over the sorted $x_{(i)}$.
Isotonic regression is the Euclidean projection of $(y_{(i)})$ onto the isotone cone and can be computed efficiently by the pool-adjacent-violators (PAV) algorithm \cite{barlow1972isotonic,robertson1988order,bestchakravarti1990activeset}.

\subsection{Unconstrained smoothing baselines (kernel bandwidth $h$, spline basis size $k$)}\label{app:c2:splines}
This section lists two standard 1D smoothers without monotonicity constraints.

\paragraph{Kernel regression (bandwidth $h$).}
The Nadaraya--Watson estimator is
\begin{equation}
\label{eq:nw}
\hat g_h(x)
=
\frac{\sum_{i=1}^n K\left(\frac{x-x_i}{h}\right)y_i}
{\sum_{i=1}^n K\left(\frac{x-x_i}{h}\right)},
\qquad h>0,
\end{equation}
where $K$ is a kernel (e.g., a compact ``cubic''/triweight kernel $K(u)=\frac{35}{32}(1-u^2)^3\mathbf{1}\{|u|\le 1\}$). The smoothing level is controlled by $h$: smaller $h$ yields less bias and more variance, and vice versa \cite{nadaraya1964,watson1964,wandjones1995}.

\paragraph{Cubic $P$-splines (basis dimension $k$ and smoothing parameter $\lambda$).}
Let $B_1,\dots,B_k$ be a cubic B-spline basis on $[0,1]$ (with chosen knots), and write
$g(x)=\sum_{j=1}^k \beta_j B_j(x)$. A standard penalized least squares problem is, for any $\lambda\ge 0$,
\begin{equation}
\label{eq:pspline}
\min_{\beta\in\mathbb{R}^k}
\frac12\sum_{i=1}^n w_i\Bigl(y_i-\sum_{j=1}^k \beta_j B_j(x_i)\Bigr)^2
+\frac{\lambda}{2}\|D^{(2)}\beta\|_2^2,
\end{equation}
where $D^{(2)}$ is a second-difference matrix on coefficients. Here $k$ controls the maximum flexibility (resolution) and $\lambda$ controls the effective smoothness \cite{eilers1996psplines}.

\subsection{Monotone $C^2$ smoothing (kernel pre-smoothing + monotone cubic splines)}\label{app:C2:monotone}
We describe a practical and stable construction yielding a globally nondecreasing and smooth ($C^2$) calibration map.

\paragraph{Step A (kernel pre-smoothing with bandwidth $h$).}
Fix an increasing grid $t_1<\cdots<t_k$ in $[0,1]$ (e.g., evenly spaced). Define pseudo-observations
\begin{equation}
\label{eq:pseudo}
\tilde y_j = \hat g_h(t_j)
=
\frac{\sum_{i=1}^n K\left(\frac{t_j-x_i}{h}\right)y_i}
{\sum_{i=1}^n K\left(\frac{t_j-x_i}{h}\right)},
\qquad j=1,\dots,k,
\end{equation}
and ``information'' weights (kernel mass)
\begin{equation}
\label{eq:mass}
v_j = \sum_{i=1}^n K\left(\frac{t_j-x_i}{h}\right).
\end{equation}
This step is purely nonparametric and controlled by the bandwidth $h$.

\paragraph{Step B (monotone cubic spline with basis size $k$ and smoothing $\lambda$).}
Let $B_1,\dots,B_k$ be a cubic B-spline basis on $[0,1]$ and write $s(x)=\sum_{\ell=1}^k \beta_\ell B_\ell(x)$ (hence $s\in C^2$).
A common way to enforce global monotonicity is to impose linear inequality constraints on coefficients (or their differences), abstractly written as
\begin{equation}
\label{eq:linmono}
A\beta \succeq 0,
\end{equation}
where $A$ encodes a sufficient condition for $s'(x)\ge 0$ on $[0,1]$ (e.g., using I-spline / integrated M-spline parametrizations or difference constraints in monotone P-spline constructions).
We then solve the convex quadratic program
\begin{align}
&\min_{\beta\in\mathbb{R}^k}~
\frac12\sum_{j=1}^k v_j\Bigl(\tilde y_j-\sum_{\ell=1}^k \beta_\ell B_\ell(t_j)\Bigr)^2
+\frac{\lambda}{2}\|D^{(2)}\beta\|_2^2 \notag \\
&\quad\text{s.t.}\quad A\beta\succeq 0\text{ and }\lambda\ge 0.\label{eq:mono_qp}
\end{align}
The resulting map is $g(x)=s(x)$, optionally clamped to $[0,1]$.

\paragraph{Probability range via a logit link (recommended for calibration).}
To guarantee $g(x)\in(0,1)$, we fit a monotone spline on the logit scale:
\begin{equation}
\label{eq:logit}
\eta(x)=\sum_{\ell=1}^k \beta_\ell B_\ell(x),~~ g(x)=\text{sgmd}(\eta(x))=\frac{1}{1+e^{-\eta(x)}}.
\end{equation}
Since $\text{sgmd}$ is strictly increasing, monotonicity of $\eta$ implies monotonicity of $g$.

\paragraph{Optimization / solver notes.}
\begin{itemize}
\item \textbf{Gaussian/least-squares case:} \eqref{eq:mono_qp} is a convex QP with linear inequalities.
It admits global solutions (no local minima). Uniqueness holds if the quadratic term is strictly convex (e.g., effective design + penalty yields a positive definite Hessian on the feasible set).
Efficient solvers include active-set QP and interior-point methods.
\item \textbf{Binomial/logit case:} one typically minimizes a penalized negative log-likelihood under the same linear constraints on $\beta$, solved by \textbf{constrained IRLS}: each IRLS iteration solves a QP of the form \eqref{eq:mono_qp} with updated working responses and weights. Shape-constrained GAM implementations follow this template \cite{pyawood2015scam}.
\end{itemize}

\paragraph{Simple properties (easy to state/derive).}
\begin{itemize}
\item \textbf{Global monotonicity:} by construction $A\hat\beta\succeq 0$ implies $s'(x)\ge 0$ (or $\eta'(x)\ge 0$), hence $g$ is nondecreasing on $[0,1]$.
\item \textbf{$C^2$ smoothness:} cubic splines are $C^2$ by construction; consequently $g$ is $C^2$ (on the link scale, and $g=\text{sgmd}\circ\eta$ is smooth as well).
\item \textbf{Convexity and global optimality (least squares):} \eqref{eq:mono_qp} is convex with linear constraints; any feasible KKT point is globally optimal.
\item \textbf{Limiting regimes:} increasing $\lambda$ shrinks curvature, approaching an (almost) affine monotone map; increasing $k$ increases representational capacity, while $h$ controls the upstream kernel smoothness.
\end{itemize}

\subsection{Simulations}

In order to illustrate, we generated some datasets $(y_i,x_{1,i},x_{2,i})$, where $x_1$ and $x_2$ are obtained from independent $\mathcal{U}([0,1])$ random variables, where $y$ is the realization of a Bernoulli random variable with mean
$$
\eta(x_1,x_2)=\displaystyle{\frac{\exp[{x_1+x_2}+\psi(x_1,x_2)]}{1+\exp[{x_1+x_2}+\psi(x_1,x_2)]}},
$$
with some non-linear function $\psi$, (here $\psi(x_1,x_2)=\exp((x_1-x_2)^3)-1$). A plain logistic regression is fitted, $\widehat{\pi}$, on the left of Figure~\ref{fig:simulation:logistic}. Then three recalibration mapping are estimated
\begin{itemize}
\setlength{\itemsep}{0.12em}
\setlength{\topsep}{0.12em}
    \item[-] piecewise constant isotonic (Section~\ref{app:isotonic}), dark green,
    \item[-] standard $\mathcal{C}^2$ splines (Section~\ref{app:c2:splines}), dark red,
    \item[-] monotonic $\mathcal{C}^2$ splines  (Section~\ref{app:C2:monotone}), dark blue.
\end{itemize}

\begin{figure}[!htb]
    \centering
    \includegraphics[width=.245\linewidth]{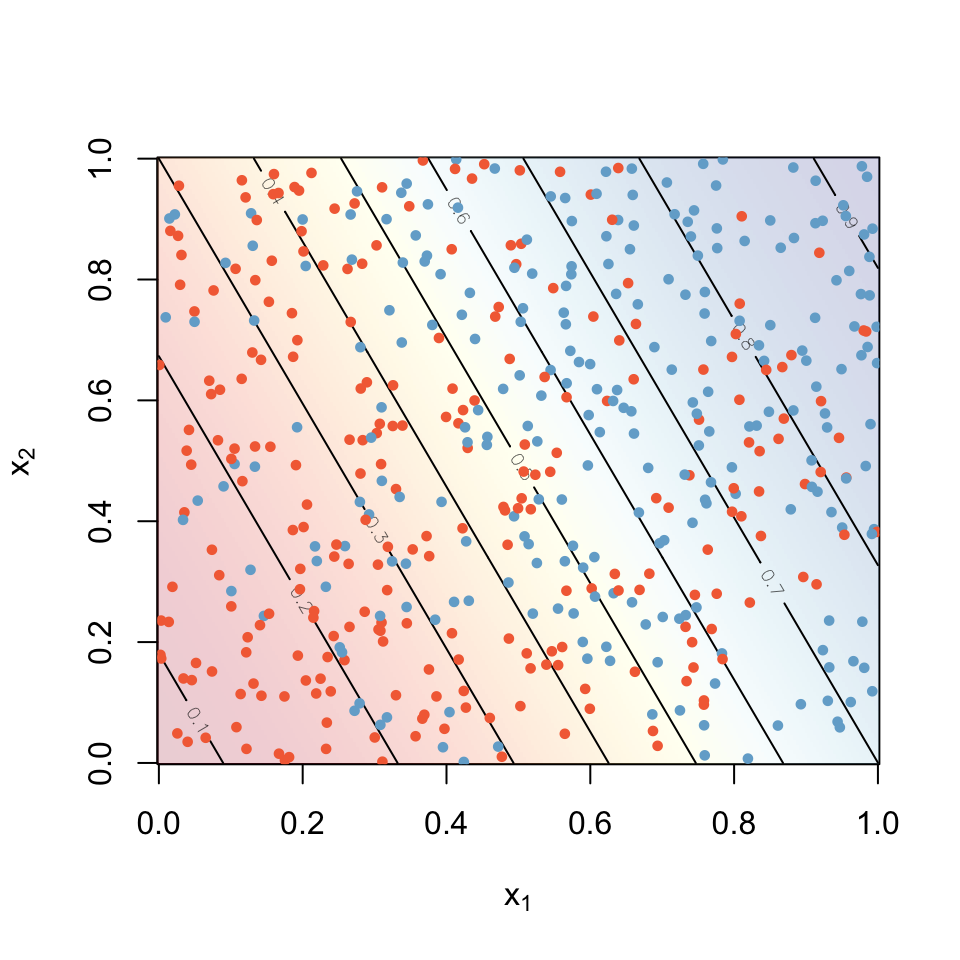}
    \includegraphics[width=.245\linewidth]{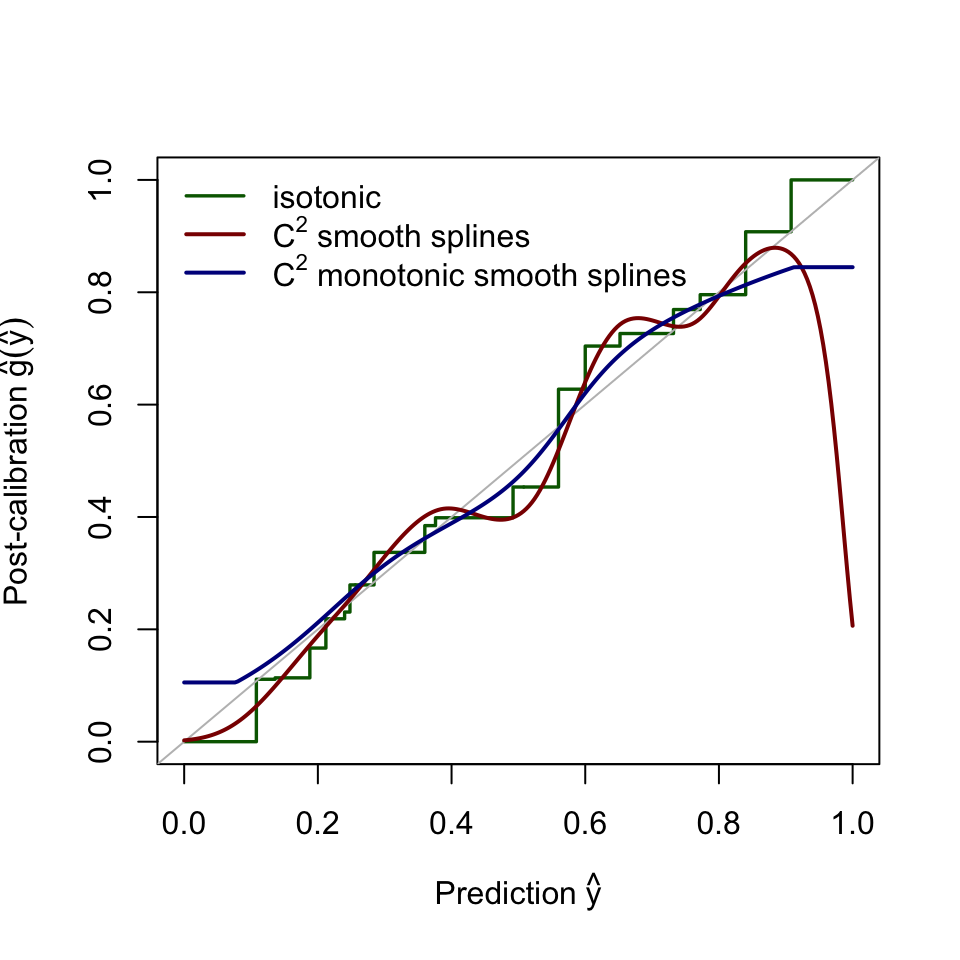}
    \includegraphics[width=.245\linewidth]{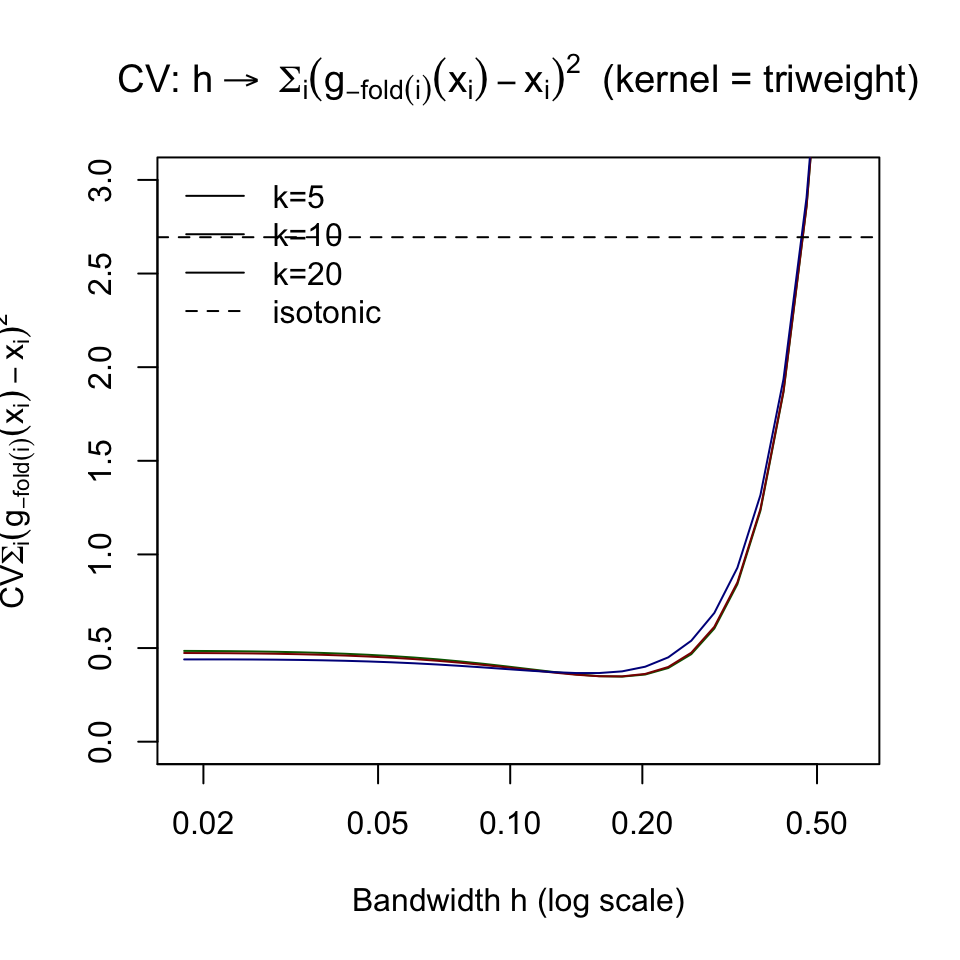}
    \includegraphics[width=.245\linewidth]{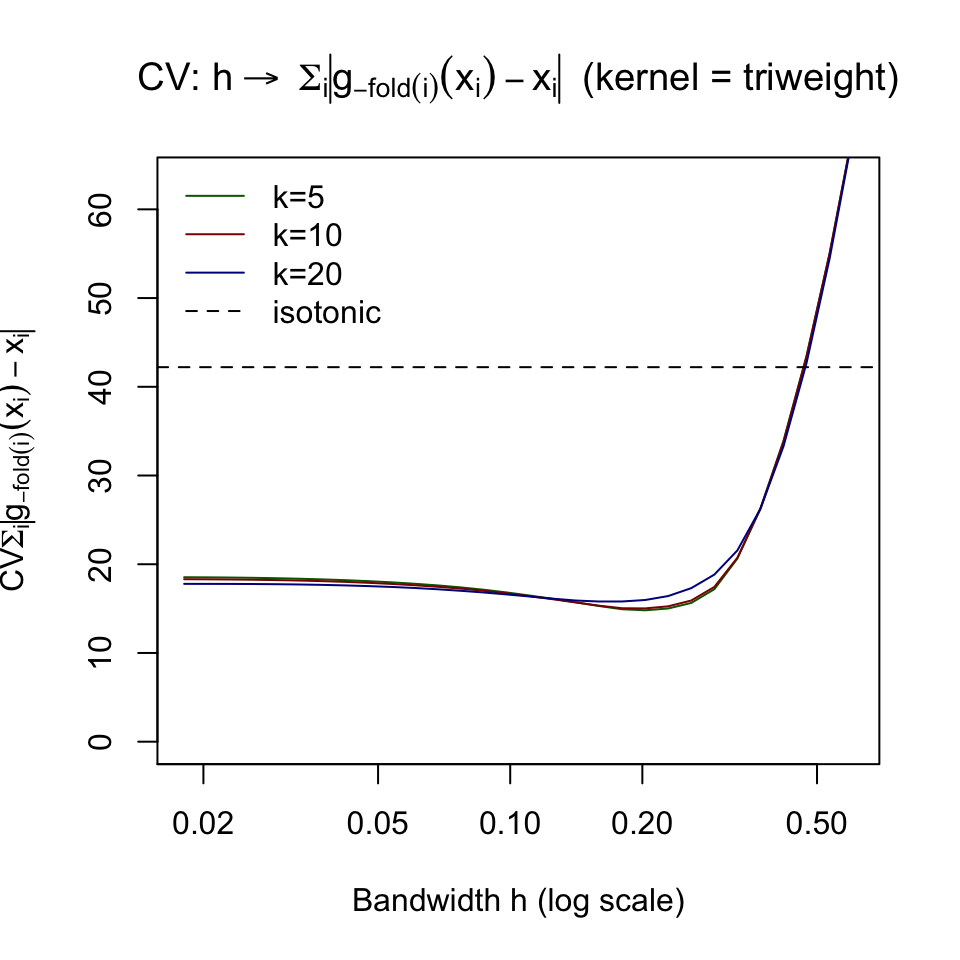}
    \caption{Plain-logistic model fitted on simulated data $\{(x_{1,i},x_{2,i},y_i)\}$, on the left, and two estimations $\hat{g}$, using an isotonic regression (in green), a $\mathcal{C}^2$ smoothing mapping (in red) and a monotone $\mathcal{C}^2$ smoothing mapping (in blue), on the middle left. On the right, evolution of $h\mapsto$ integrated calibration index ($\ell_1$) on the middle right, and local calibration score ($\ell_2$) for $\hat{g}_h$ on the right, as a function of the bandwidth parameter $h>0$, for three values of $k$, and using 5-fold cross validation}
    \label{fig:simulation:logistic}
\end{figure}

To check for calibration, classical calibration metrics are
\begin{align*}
   \text{Brier Score } 
   & = \frac{1}{n}\sum_{i=1}^{n} \big(\hat{s}(\mathbf{x}_i) - y_i\big)^{2} \\
   \text{Integrated Calibration Index } 
   & = \frac{1}{n}\sum_{i=1}^{n} 
   \big|\hat{s}(\mathbf{x}_i) - \hat{g}_h\big(\hat{s}(\mathbf{x}_i)\big)\big| \\
   \text{Local Calibration Score } 
   & = \frac{1}{n}\sum_{i=1}^{n} 
   \big[\hat{s}(\mathbf{x}_i) - \hat{g}_h\big(\hat{s}(\mathbf{x}_i)\big)\big]^2 
\end{align*}
see \cite{brier_1950,Austin2019TheIC,pmlr-v119-zhang20k,brier-use-2021}. 


\end{document}